\documentclass[10.5pt,letterpaper]{article}

\usepackage{graphicx} % more modern
\graphicspath{ {images/} }

\usepackage{caption}
\usepackage{subcaption}
\usepackage{amsmath,amssymb,amsthm}
\usepackage{placeins}
\usepackage{color,mathtools}
\usepackage{bm}      % for \bm macro
\usepackage{pifont}% http://ctan.org/pkg/pifont
\usepackage{cite}
\usepackage{multirow}

\usepackage{multicol} %for figures in appendix;
  {\par\medskip\noindent\minipage{\linewidth}}
  {\endminipage\par\medskip}


  
\usepackage{algorithm}
\usepackage{algorithmic}
\usepackage{hyperref}

\usepackage{fullpage}

\usepackage{mathtools}

\newcommand{\bff}{{\bf{f}}}
\newcommand{\bfh}{{\bf{h}}}
\newcommand{\Q}{\bm{Q}}

\newcommand{\figref}[1]{Fig.~\ref{fig:#1}}  % use for citing figs
\newcommand{\secref}[1]{Sec.~\ref{sec:#1}}  % use for citing secs
\newcommand{\punt}[1]{}

\newcommand{\ch}{\mathcal{H}}

\newcommand{\bfx}{\mathbf{x}}

\newcommand{\Dat}{\mathcal{D}}

\newcommand{\Nrm}{\mathcal{N}}

\newcommand{\mSigma}{\mathbf{\Sigma}}
\newcommand{\mPsi}{\mathbf{\Psi}}
\newcommand{\mC}{\mathbf{C}}

\newcommand{\mD}{\mathbf{D}}
\newcommand{\mQ}{\mathbf{Q}}
\newcommand{\mM}{\mathbf{M}}

\newcommand{\mLambda}{\mathbf{\Lambda}}
\newcommand{\veta}{\mathbf{\ensuremath{\bm{\eta}}}}

\newcommand{\mI}{\mathbf{I}}
\newcommand{\trp}{^\top}

\newcommand{\vphi}{\mathbf{\ensuremath{\bm{\phi}}}}

\newcommand{\vw}{\mathbf{w}}
\newcommand{\vx}{\mathbf{x}}

\newcommand{\vz}{\mathbf{z}}

\newcommand{\vv}{\mathbf{v}}

\newcommand{\vy}{\mathbf{y}}
\newcommand{\vu}{\mathbf{u}}

\newcommand{\vn}{\mathbf{n}}
\newcommand{\vm}{\mathbf{m}}
\newcommand{\bq}{\begin{equation}}
\newcommand{\eq}{\end{equation}}
\newcommand{\ba}{\begin{eqnarray}}
\newcommand{\ea}{\end{eqnarray}}

\newcommand{\remove}[1]{}

\newcommand{\xv}{\bm{x}}

\newcommand{\bv}{\bm{b}}
\newtheorem{theorem}{Theorem}[section]

\newtheorem{lemma}[theorem]{Lemma}

\allowdisplaybreaks[4]

\usepackage[square,sort,comma,numbers]{natbib}
\title{A Differentially Private Kernel Two-Sample Test}
\author{Anant Raj\thanks{Equal Contribution} \\
  MPI-IS \\
  \texttt{anant.raj@tuebingen.mpg.de} 
   \and
   Ho Chung Leon Law\footnotemark[1] \\
   University Of Oxford \\
   \texttt{ho.law@stats.ox.ac.uk} 
   \and 
   Dino Sejdinovic \\
   University Of Oxford \\
   \texttt{dino.sejdinovic@stats.ox.ac.uk} 
   \and
   Mijung Park \\
   MPI-IS\\
   \texttt{mijung.park@tuebingen.mpg.de} 
   }

\date{April 17, 2018}

\begin{document}
\maketitle
\begin{abstract}
% kernel two-sample testing
Kernel two-sample testing is a useful statistical tool in determining whether data samples arise from different distributions without imposing any parametric assumptions on those distributions. 
%It can be used, e.g., to determine whether experimental data measured in different scientific studies can be analyzed jointly.
% Privacy
However, raw data samples can expose sensitive information about individuals who participate in scientific studies, which makes the current tests vulnerable to privacy breaches. 
Hence, we design a new framework for kernel two-sample testing conforming to differential privacy constraints, in order to guarantee the privacy of subjects in the data. 
Unlike existing differentially private parametric tests that simply add noise to data, kernel-based testing imposes a challenge due to a complex dependence of test statistics on the raw data, as these statistics correspond to estimators of distances between representations of probability measures in Hilbert spaces.  
%
% analytic 
%We overcome this by taking distances in terms of analytic functions which map the space of probability measures into a finite dimensional Euclidean space. 
%We consider an approach
%an approach to kernel-based testing,
%applicable to large-scale problems, 
Our approach considers finite dimensional approximations to those representations. 
As a result, a simple chi-squared test is obtained, where a test statistic depends on a mean and covariance of empirical differences between the samples, 
%and a covariance matrix of those differences, 
which we perturb for a privacy guarantee. %
%We perturb either mean and covariance, or the statistic directly, in order to guarantee privacy of the underlying data. 
% 
We investigate the utility of our framework  in two realistic settings
%two different scenarios: the {\it{trusted-curator}} setting and the {\it{no-trusted-entity}} setting.    
%
%We empirically demonstrate that 
and conclude that our method requires only a relatively modest increase in sample size to achieve a similar level of power to the non-private tests in both settings.
\end{abstract}
\section{Introduction}
Several recent works suggest that it is possible to identify subjects that have participated in scientific studies based on publicly available aggregate statistics (cf. \cite{Homer2008, Johnson2013} among many others). The {\it{differential privacy}} formalism \cite{dwork2006calibrating} provides a way to quantify the amount of information on whether or not a single individual's data is included (or modified) in the data and also provides rigorous privacy guarantees in the presence of \textit{arbitrary side information}.

An important tool in statistical inference is {\it{two-sample testing}}, in which samples from two probability distributions are compared in order to test the null hypothesis that the two underlying distributions are identical against the general alternative that they are different. %This problem often arises when determining whether it is appropriate to jointly analyze data from several experimental studies (e.g. performed by different laboratories) or whether systematic differences in data distributions might have been introduced by different experimental procedures. 
In this paper, we focus on nonparametric, {\it{kernel-based}} two-sample testing approach and investigate the utility of this framework in a differentially private setting. The kernel-based two-sample testing was introduced by Gretton et al \cite{Gretton2006,Gretton2012jmlr} who considers an estimator of maximum mean discrepancy (MMD) \cite{Borgwardt2006}, the distance between embeddings of probability measures in a reproducing kernel Hilbert space (RKHS) (See \cite{Muandet2017} for a recent review), as a test statistic for the nonparametric two-sample problem. %This framework is applicable to any data domains on which positive definite kernels can be defined, multivariate or high-dimensional data, as well as structured or non-Euclidean spaces, such as strings or graphs.

Many existing differentially private testing 
%(or variants of differentially private testing) 
methods are based on categorical data, i.e. counts \cite{Gaboardi2016, Gaboardi17, rogers2017new}, in which case a natural way to achieve privacy is simply adding noise to these counts. However, when we consider a more general input space $\mathcal X$, the amount of noise needed to privatise the data essentially becomes the order of diameter of the input space (Details in the Appendix~\ref{app:per_samples}). For spaces such as $\mathbb{R}^d$, the noise that needs to be added can destroy the utility of the data, and hence also for the test. Hence, we take an alternative approach, and privatise only quantities that are required for the test, as in general we require less noise for the differential privacy of summary statistics of the data. For testing, we only require the empirical kernel embedding $\frac{1}{n}\sum_{i}k(\vx_i,\cdot)$ corresponding to a dataset, where $\vx_i\in\mathcal X$ and $k$ is some positive definite kernel (discussed in Appendix~\ref{app:rkhs_noise}). Now, since kernel embedding lives in $\mathcal H_k$, a space of functions, a natural way to protect them is to add Gaussian Process noise \cite{hall2013differential}. Although sufficient for situations where the functions themselves are of interest, embeddings impaired by a Gaussian process does not lie in the same RKHS \cite{Wahba90a}, and hence one cannot estimate RKHS distances between such noisy embeddings.
Alternatively, one could consider adding noise to an estimator of MMD \cite{Gretton2012jmlr}. However, asymptotic null distributions of these estimators are data dependent and the test thresholds are typically computed by permutation testing or by eigendecomposing centred kernel matrices of the data \cite{Gretton2009_spectral}. In this case neither of these approaches is available in a differentially private setting as they both require further access to data.

In this paper, we build a differentially private two-sample testing framework, by considering {\it{analytic representations}} of probability measures \cite{chwialkowski2015fast,Jitkrittum2016} aimed at large scale testing scenarios. As a result, we are able to obtain a test statistic that is based on means and covariance of feature vectors of the data. Here, the asymptotic distribution under the null hypothesis of the test statistic does not depend on the data, making this framework a convenient choice for differential privacy. With this setup, we will consider two approaches: 1) add noise to these mean or covariances 2) add noise to the statistic itself. 

We now present the two privacy scenarios that we consider and also motivate their usage. In the first scenario, we assume there is a trusted curator and also an untrusted tester, in which we want to protect data from. In this setting, the trusted curator has access to the two datasets and computes the mean and covariance of the empirical differences between the feature vectors. The curator can protect the data in two different ways: (1) perturb mean and covariance separately and release them; or (2) compute the statistic without perturbations and add noise to it directly. The tester takes these perturbed quantities and performs the test at a desired significance level. Here, we separate the entities of truster and curator, as it is rarely that a non-private decision whether to reject or not is of interest, for example tester may require test-statistic/p-values for multiple hypothesis testing corrections.
In the second scenario, we assume that there are two data-owners, each having one data sample, and a tester, and none of the parties trust each other. Each data-owner now has to perturbs their own mean and covariance of the feature vectors and release them to the tester. 

Under each setting, we exploit various differentially private mechanisms and empirically study the utility of the proposed framework. In particular, we demonstrate that while the asymptotic null distributions remain unchanged under the differentially private scenario, extra caution needs to be exercised when resorting to such asymptotics. Unlike the non-private case, using the asymptotic null distribution to compute p-values can lead to grossly miscalibrated Type I control. We propose a remedy for this problem, and give approximations of the finite-sample null distributions, yielding good Type I control and power-privacy tradeoffs experimentally in \secref{experiments}.

While there are several works that connect kernel methods with differential privacy, including \cite{Jain13,hall2013differential,balog2017privacy}, this is, to the best of our knowledge, the first attempt to make the kernel-based two-sample testing procedure differentially private.
%and demonstrate experimentally 

%which experimentally  
%yield correct Type I control and good power-privacy tradeoffs  %\leon{Maybe we should rephrase power part, because I do not have results for power for the asymptotics in paper, and the power is good on asymptotics, just because the Type I error is so off anyways....}
%

%
We start by providing a brief background on kernel two-sample test using analytic representation and on differential privacy and introduce the two privacy settings we consider in this paper in \secref{background}.  We derive essential tools for the proposed test in \secref{methods_1} and \secref{methods_2}, and describe approximations to finite-sample null distributions in \secref{null_analysis}. We illustrate the effectiveness of our algorithm in \secref{experiments}.

\section{Background}\label{sec:background}
In this section, we provide background information on  kernel two-sample test using analytic representation and the definition of  algorithmic privacy that we will use in our algorithm.
\paragraph{Mean embedding and smooth characteristic function tests}
First introduced by \cite{chwialkowski2015fast} and then extended and further analyzed by \cite{Jitkrittum2016}, these two tests are state-of-the-art kernel-based testing approaches applicable to large datasets. Here, we will focus on the approach by \cite{Jitkrittum2016}, and in particular on the  mean embedding (ME) and on characterization based on the smooth characteristic function (SCF). Assume that we observe samples $\{\vx_i\}_{i=1}^n\sim P$ and $\{\vy_i\}_{i=1}^n\sim Q$, where $P$ and $Q$ are some probability measures on $\mathbb{R}^D$.  We wish to test the null hypothesis ${\bf H}_0: P=Q$ against all alternatives. Both ME and SCF tests consider finite-dimensional feature representations of the empirical measures $P_n$ and $Q_n$ corresponding to the samples $\{\vx_i\}_{i=1}^n \sim P$ and $\{\vy_i\}_{i=1}^n\sim Q$ respectively. 
The ME test considers feature representation given by
$ \vphi_{P_n} = \frac{1}{n}\sum_{i = 1}^n \left[  k(\vx_i, T_1), \cdots,  k(\vx_i, T_J)  \right] \in \mathbb{R}^J,$
for a given set of test locations $\{T_j\}_{j=1}^J$, i.e. it evaluates the kernel mean embedding $\frac{1}{n}\sum_{i = 1}^n k(\vx_i,\cdot)$ of $P_n$ at those locations.
We write $\vw_n = \vphi_{P_n}  - \vphi_{Q_n} $ to be the difference of the feature vectors of the empirical measures $P_n$ and $Q_n$.
%is given by :
%\begin{align}
%\vw_n =  \Big[ \frac{1}{n}\sum_{i = 1}^n \big(k(\vx_i, T_1) - k(\vy_i, T_1)\big), \cdots , \frac{1}{n}\sum_{i = 1}^n \big(k(\vx_i, T_J) - k(\vy_i, T_J)\big)  \Big] \label{eq:mean_vec}
%\end{align}
If we write $\vz_i =\Big[ k(\vx_i, T_1) - k(\vy_i, T_1), \cdots , k(\vx_i, T_J) - k(\vy_i, T_J)\Big]$, then 
$\vw_n = \frac{1}{n}\sum_{i=1}^n \vz_i$.  We also define the empirical covariance matrix 
$\mSigma_n = \frac{1}{n-1}\sum_{i=1}^n (\vz_i-\vw_n) (\vz_i-\vw_n)\trp.$
%We want to add noise separately to the mean $\vw_n$ and the covariance $\Sigma_n$  to make 
The final statistic is given by 
\begin{align}
\label{eq:teststatistic}
    s_n = n~\vw_n^\top(\Sigma_n+\gamma_n I)^{-1}\vw_n,
\end{align}
where, as \cite{Jitkrittum2016} suggest, a regularization term $\gamma_n I$ is added onto the empirical covariance matrix for numerical stability. This regularization parameter will also play an important role in analyzing sensitivity of this statistic in a differentially private setting.
Following \cite[Theorem 2]{Jitkrittum2016}, one should take $\gamma_n\to 0$ as $n\to \infty$, and in particular, $\gamma_n$ should decrease at a rate of $\mathcal O(n^{-1/4})$. The SCF setting uses the statistic of the same form, but considers features based on empirical characteristic functions~\cite{rahimi2008random}.  \\ 

Thus, it suffices to set $\vz_i \in \mathbb R^{J}$ to %\vspace{-3mm}
%, i.e.
%\begin{align}
 % \vphi_{P_n} = \frac{1}{n}\sum_{i = 1}^n \Big[  h(\vx_i)\cos(\vx_i^\top T_1), h(\vx_i)\sin(\vx_i^\top T_1), \ldots, h(\vx_i)\cos(\vx_i^\top T_J), h(\vx_i)\sin(\vx_i^\top T_J)\Big], \nonumber
%\end{align}
%where real and imaginary part of the empirical characteristic functions are stacked together. 
%\begin{align}
$$\vz_i =\Big[ g(\vx_i)\cos(\vx_i^\top T_j)-g(\vy_i)\cos(\vy_i^\top T_j),  g(\vx_i)\sin(\vx_i^\top T_j)-g(\vy_i)\sin(\vy_i^\top T_j)\Big]_{j=1}^J,$$    %\end{align}
where $\{T_j\}_{j=1}^{J/2}$ is a given set of frequencies, and $g$ is a given function which has an effect of smoothing the characteristic function estimates (cf. \cite{chwialkowski2015fast} for derivation).
 The test then proceeds in the same way. For both the cases, the distribution of the test statistic \eqref{eq:teststatistic} under the null hypothesis ${\bf H}_0: P=Q$ converges to a chi-squared distribution with $J$ degrees of freedom. This follows from a central limit theorem argument whereby $\sqrt n \vw_n$ converges in law to a zero-mean multivariate normal distribution $\mathcal N(0,\mSigma)$ where $\mSigma =\mathbb E[\vz \vz^\top]$, while $\mSigma_n+\gamma_n I \to \mSigma$ in probability.

While \cite{chwialkowski2015fast} uses random distribution features, i.e. test locations/frequencies $\{T_j\}_j$ are sampled randomly from a predefined distribution, \cite{Jitkrittum2016} selects test locations/frequencies $\{T_j\}_j$ which maximize the test power, yielding interpretable differences between the distributions under consideration. Throughout the paper, we assume that  we use bounded kernels in the ME test, in particular $k(\vx,\vy)\leq \kappa/2,  ~~ \forall \vx,\vy$, and that the weighting function in the SCF test is also bounded: $h(\vx)\leq \kappa/2$ 
Hence, $||\vz_i ||_2\leq \kappa \sqrt{J}$ in both cases, for any $i\in [1,n]$. 

\paragraph{Differential privacy}
\begin{figure}
  \begin{center}
    \includegraphics[width=0.4\textwidth]{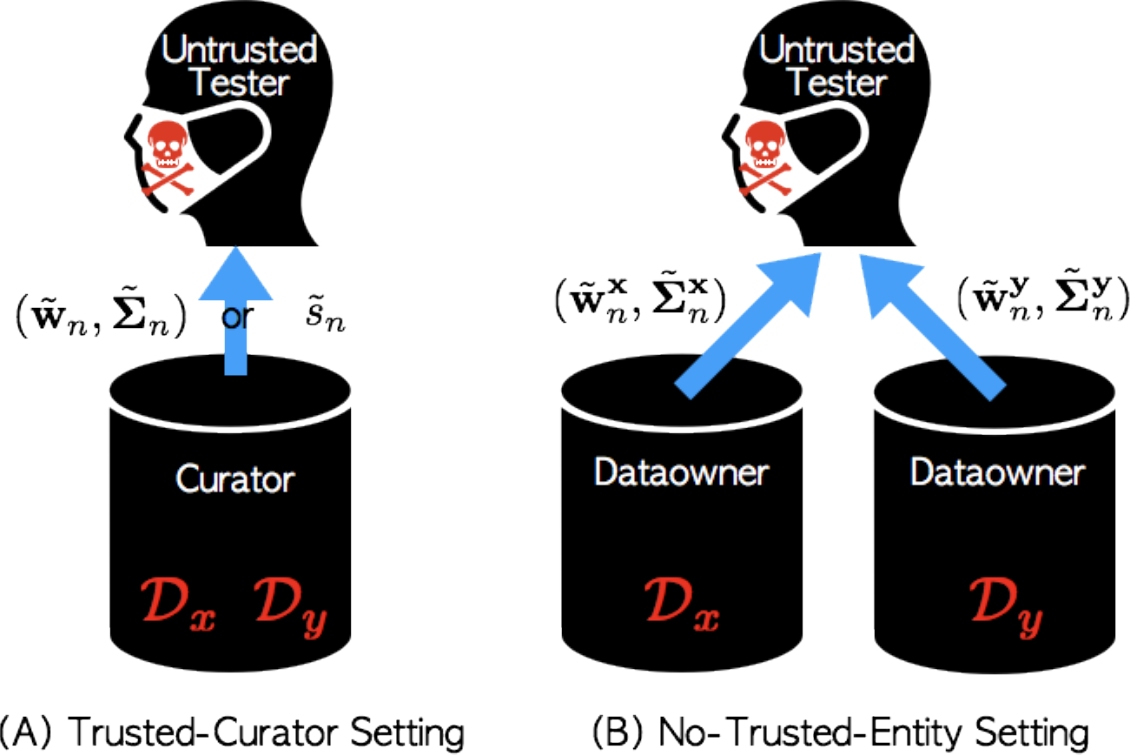}
  \end{center}
  \caption{Two privacy settings. \textbf{(A)} A trusted curator releases a private test statistic or private mean and covariance of empirical differences between the features.  \textbf{(B)} Data owners release private feature means and covariances calculated from their samples. In both cases, an untrusted tester performs a test using the private quantities.}
\label{fig:schematic}
\end{figure}

Given an algorithm $\mathcal{M}$ and neighbouring datasets $\Dat$, $\Dat'$ differing by a single entry,
the \emph{privacy loss} of an outcome $o$ is
%\vspace{-0.1cm}
%\begin{equation}
$L^{(o)} = \log \frac{Pr(\mathcal{M}_{(\Dat)} = o)}{Pr(\mathcal{M}_{(\Dat')} = o)} \mbox{ .}$
%\vspace{-0.1cm}
%\end{equation} 
%
The mechanism $\mathcal{M}$ is called $\epsilon$-DP if and only if
$|L^{(o)}| \leq \epsilon, \forall o, \Dat, \Dat'$.
A weaker version of the above is ($\epsilon, \delta$)-DP, if and only if
$|L^{(o)}| \leq \epsilon$, with probability at least $1-\delta$.
The definition states that a single individual's participation in the data do not change the output probabilities by much, which limits the amount of information that the algorithm reveals about any one individual.

A way of designing differentially private algorithms is by adding noise to the algorithms' outputs. Suppose a deterministic function $h: \Dat \mapsto \mathbb{R}^p$ computed on sensitive data $\Dat$ outputs a $p$-dimensional vector quantity. For making $h$ private, we add noise in function $h$~\cite{dwork2006our} which  is calibrated to the  {\it{global sensitivity}}, $GS_h$, of function $h$ defined by the maximum difference in terms of $L_2$-norm, $||h(\Dat)-h(\Dat') ||_2$, for neighboring $\Dat$ and $\Dat'$. In the case of Gaussian mechanism (Theorem 3.22 in \cite{Dwork14}), the output is perturbed  by ; $\tilde{h}(\Dat) = h(\Dat) + \Nrm(0, GS_h^2\sigma^2 \mathbf{I}_p)$.
The perturbed function $\tilde{h}(\Dat) $ is $(\epsilon, \delta)$-DP, where $\sigma \geq \sqrt{2\log(1.25/\delta)}/\epsilon$, for $\epsilon \in (0,1)$. 
In \secref{methods_1}, we exploit several existing differentially private mechanisms to achieve differentially private test statistics.

When constructing our tests, we use two important properties of differential privacy.
%: {\it{composability}} and {\it{post-processing invariance}}.  
The composability theorem \cite{dwork2006our} tells us that the strength of privacy guarantee degrades with repeated use of DP-algorithms. In particular, when two differentially private subroutines are combined, where each one guarantees $(\epsilon_1,\delta_1)$-DP and $(\epsilon_2,\delta_2)$-DP respectively by adding independent noise, the parameters are simply composed%\footnote{For more advanced composition methods, see \cite{zCDP16, 2016arXiv160700133A}.} 
by $(\epsilon_1+\epsilon_2, \delta_1+\delta_2)$. Furthermore, post-processing invariance \cite{dwork2006our} tells us that the composition of any arbitrary data-independent mapping with an $(\epsilon,\delta)$-DP algorithm is also $(\epsilon,\delta)$-DP.

\paragraph{Privacy settings}
%\subsection{PRIVACY SETTINGS FOR KERNEL TWO-SAMPLE TESTS}
%\label{sec:two-setting}

% I will add "schematic" later. 
We consider the two different privacy settings as shown in \figref{schematic}: \\ (A) {\bf{Trusted-curator (TC)}} setting: there is a trusted entity called curator that handles datasets and outputs the private test statistic, either in terms of perturbed $\tilde{\vw}_n$ and $\tilde{\mSigma}_n$, or in terms of perturbed test statistic $\tilde{s}_n$. An untrusted tester performs a chi-square test given these quantities.\\ (B) {\bf{No-trusted-entity (NTE)}} setting: each data owner outputs private mean and covariance of the feature vectors computed on their own dataset, meaning that the owner of dataset $\Dat_x$ outputs $\tilde{\vw}_n^{\vx}$ and $\tilde{\mSigma}_n^{\vx}$ and the owner of dataset $\Dat_y$ outputs $\tilde{\vw}_n^{\vy}$ and $\tilde{\mSigma}_n^{\vy}$.  An untrusted tester performs a chi-squared test given these quantities.

%we simply replace $J$ by $2J$, and this also holds.
%(3)  the largest L2-norm of $\vz_i$ is bounded by $||\vz_i ||_2\leq \kappa \sqrt{J}$ for any $i\in [1,n]$.

\section{Trusted-curator setting}\label{sec:methods_1}

%\subsection{Mean and covariance pertrubation}

In this setting, a trusted curator releases either a private test statistic or private mean and covariance which a tester can use to perform a chi-square test.
Given a total privacy budget ($\epsilon, \delta$), when we perturb mean and covariance separately, we spend ($\epsilon_1, \delta_1$)  for mean perturbation and ($\epsilon_2, \delta_2$) for covariance perturbation, such that $\epsilon = \epsilon_1 + \epsilon_2$ and $\delta = \delta_1 + \delta_2$.  

\subsection{Perturbing mean and covariance}
\label{sub:tcmc-sensitivity}
%\paragraph{Private Mean and Covariance: }

%= (\sqrt{n}\vw_n)^\top \Sigma_n^{-1} (\sqrt{n}\vw_n)$ 
%Below, we add noise to the mean vector and the covariance matrix separately, to ensure the resulting statistic is
%$(\epsilon, \delta)$-differentially private.

% ============== mean perturbation ++++++++++++++++++

\paragraph{Mean perturbation}
We obtain a private mean by adding Gaussian noise based on the analytic Gaussian mechanism recently proposed in \cite{balle2018improving}. The main reason for using this Gaussian mechanism over the original \cite{Dwork14} is that it provides a DP guarantee with smaller noise. 
%\begin{theorem}[Gaussian mechanism]\label{lem:dwork}
%\begin{lemma}[Gaussian mechanism taken from \cite{dwork2006our,dwork2006calibrating}] %
For $\vw_n : \Dat \rightarrow \mathbb{R}^J$ that has the global L2-sensitivity $GS_2(\vw_n)$, the analytic Gaussian mechanism produces 
%\begin{align}
$\tilde{\vw}_n(\Dat) = \vw_n(\Dat) + \vn, \quad \mbox{ where } \vn { \sim } \Nrm (\mathbf{0}_J, \sigma_{\vn}^2 \mI_{J \times J})
$. 
Then $\tilde{\vw}_n(\Dat)$ is $(\epsilon_1, \delta_1)$-differentially private mean vector if $\sigma_\vn$ follows the regime in Theorem 9 of \cite{Dwork14}\footnote{We utilise the author's code available at https://github.com/BorjaBalle/analytic-gaussian-mechanism}, here implicitly $\sigma_\vn$ depends on $GS_2(\vw_n), \epsilon_1$ and $\delta_1$.
%\nonumber \label{eq:private_mean}
%
%is $\epsilon$-differentially private if $\mathcal{D} = Laplace\Big( \frac{GS_1(\phi)}{\epsilon} \Big)$ or 
%\end{align}
%$(\epsilon_1, \delta_1)$-differentially private mean vector if $\sigma_\vn \geq \frac{GS_2(\vw_n)\sqrt{2 \log{(1.25/\delta_1)}}}{\epsilon_1}$.
%\end{lemma}
%\end{theorem}
%The proof follows the proof of Theorem 3.22 in \cite{Dwork14}, 
%
Assuming an entry difference between two paris of datasets $\Dat = (\Dat_x, \Dat_y)$ and $\Dat'=(\Dat'_x, \Dat'_y)$
the global sensitivity is simply
%\footnote{We assume one entry difference between two paris of datasets $(\Dat_x, \Dat_y)$ and $(\Dat'_x, \Dat'_y)$. }, 
%given by 
%We denote the output of this mechanism, i.e., the noised-up mean vector by $\tilde{\vw}_n$. 
%
%\subsection{Sensitivity Analysis of $\vw_n$} \label{subsec:analyis_mean}
%
% 
%we derive the global sensitivity of $\vw_n$ as
\begin{align}
GS_2(\vw_n) = \max_{\Dat, \Dat'} \| \vw_n(\Dat) - \vw_n(\Dat') \|_2~ =  \max_{\vz_n, \vz'_n} \tfrac{1}{n} \| \vz_n- \vz'_n\|_2 \leq \tfrac{\kappa \sqrt{J}}{n}.
\label{eq:mean_sensitivity}
\end{align}
%
%Hence, the sensitivity of $\sqrt{n}\vw_n$ is $\frac{\kappa \sqrt{J}}{\sqrt{n}}$.
%The following lemma~\ref{lem:dwork} from \cite{dwork2006our,dwork2006calibrating}, tells the amount of noise to be added in $\vw_n$ for making it $(\epsilon_1,\delta_1)$-differentially private.
%
%
% ============== covariance perturbation ++++++++++++++++++
%
\paragraph{Covariance perturbation}
To obtain a private covariance, we consider \cite{DworkTT014} which utilises Gaussian noise. Here since the covariance matrix is given by $\mSigma_n = \mLambda - \frac{n}{n-1}\vw_n \vw_n\trp$, where $\mLambda  = \frac{1}{n-1} \sum_{i =1}^n \vz_i \vz_i\trp$, we can simply privatize the covariance by simply perturbing the 2nd-moment matrix $\mLambda$ and using the private mean $\tilde{\vw}_n$, i.e.,  $\tilde{\mSigma}_n = \tilde{\mLambda} - \frac{n}{n-1}\tilde{\vw}_n\tilde{\vw}_n\trp$. To construct the 2nd-moment matrix $\tilde{\mLambda}$ that is $(\epsilon_2, \delta_2)$-differentially private, we use $\tilde{\mLambda} = \mLambda + \mPsi$, where $\mPsi$ is obtained as follows:
\begin{enumerate}
\item Sample from $\veta \sim \Nrm(0, \beta^2 \mI_{J(J+1)/2})$, where $\beta$ is a function of global sensitivity $GS(\mLambda), \epsilon_2, \delta_2$, outlined in Theorem \ref{thm:anaylzegauss} in the appendix.
%= \frac{\kappa^2 J}{(n-1)}$.
\item Construct an upper triangular matrix (including diagonal) with entries from $\veta$.
\item Copy the upper part to the lower part so that resulting matrix $\mPsi$ becomes symmetric.
\end{enumerate}
Now using the composability theorem \cite{dwork2006our} gives us that $\tilde{\mSigma}_n$ is $(\epsilon, \delta)$-differentially private.

\subsection{Perturbing test statistic}
\label{sec:test_statistic}
%\paragraph{Analysis of $s_n$: }
%Since we add independent noise to the mean vector and the 2nd-moment matrix, the resulting test statistic is $(\epsilon_1+\epsilon_2, \delta_1+\delta_2)$-differentially private. 

%\subsection{Sensitivity Analysis of Statistics $\vw_n^\top \mSigma^{-1} \vw_n$}
The trusted-curator can also release a differentially private statistic, to do this we use the analytic Gaussian mechanism as before, perturbing the statistic by adding Gaussian noise. To use the mechanism, we need to calculate the global sensitivity needed of the test statistic $s_n = \vw_n^\top (\mSigma_n+\gamma_n I)^{-1} \vw_n$, which we provide in this Theorem(proof can be found in Appendix \ref{app:sens_stats}):
%To adjust the noise variance accordingly, we analyze the sensitivity  of the test statistics $s_n = \vw_n^\top (\mSigma_n+\gamma_n I)^{-1} \vw_n$. 
%The following theorem provides the global sensitivity of the test statistic, with the proof provided in Appendix ~\ref{app:sens_stats}. 
%Getting the sensitivity of $s_n$ is crucial in the case where one would directly want to add noise in the $s_n$. Our assumptions on the kernel and the data remains the same from the previous sections.  Since, $s_n$ is just a number, hence $GS_1$ and $GS2$ are same. Here below in theorem

\begin{theorem}\label{thm:test_sn}
Given the definitions of $\vw_n$ and $\mLambda_n$, and the L2-norm bound on $\vz_i$'s, the global sensitivity $GS_2(s_n)$ of the test statistic $s_n$ is
$\frac{4\kappa^2 J \sqrt{J} }{n \gamma_n } \left( 1 + \frac{ \kappa^2 J}{n-1} \right)$,
%The new sensitivity I would like to write is 
%\frac{4 \kappa^2 J^{3/2}}{n} \left(\frac{\kappa^2 J}{n} + 1\right)\gamma $
where $\gamma_n$ is a regularization parameter.%\footnote{For an ill-conditioned 2nd-moment matrix, we add a regularizing constant $\gamma_n$, which scales with the data size $n$ as explained in \secref{background}. } 
which we set to be smaller than the smallest eigenvalue of $\mLambda$.
\end{theorem}
%\dino{as discussed with Leon, having this $\mu_{min}(\mLambda)- B^2$ is a bad idea - why not just keep it as $\mu_\min(M_\lambda)$ - since we are regularising anyway, no point going from (40) to (41) since $\mu_\min(M_\lambda)\geq \gamma_n$ and we are just creating a looser inequality!}
%\begin{proof}
%Proof is given in the Appendix~\ref{app:sens_stats}.
%\end{proof}

\section{No-trusted-entity setting}\label{sec:methods_2}
In this setting, the two samples $\{\vx_i\}_{i=1}^{n_\vx}\sim P$ and $\{\vy_j\}_{j=1}^{n_\vy}\sim Q$ reside with different data owners each of which wish to protect their samples in a differentially private manner. Note that in this context we allow the size of each sample to be different. The data owners first need to agree on the given kernel $k$ 
%(we describe below a method which allows the data owners to agree on an appropriate kernel parameter in a differentially private manner) 
as well as on the test locations $\{T_j\}_{j=1}^J$. 
We denote now 
$\vz^\vx_i =\Big[ k(\vx_i, T_1), \cdots , k(\vx_i, T_J) \Big]\trp$ in the case of the ME test or
$\vz^\vx_i =\Big[ h(\vx_i)\cos(\vx_i^\top T_j),  h(\vx_i)\sin(\vx_i^\top T_j)\Big]_{j=1}^J$ in the case of the SCF test. Also, we denote
$\vw_{n_\vx}^\vx = \frac{1}{n_\vx}\sum_{i=1}^n \vz^\vx_i$, 
$\mSigma_{n_\vx}^\vx= \frac{1}{{n_\vx}-1}\sum_{i=1}^{n_\vx} (\vz^\vx_i-\vw_{n_\vx}^\vx) (\vz^\vx_i-\vw_{n_\vx}^\vx)\trp$, and similarly for the sample $\{\vy_j\}_{j=1}^{n_\vy}\sim Q$.
The respective means and covariances $\vw_{n_\vx}^\vx$, $\mSigma_{n_\vx}^\vx$ and $\vw_{n_\vy}^\vy$, $\mSigma_{n_\vy}^\vy$ are computed by their data owners, which then impair them independently with noise according to the sensitivity analysis described in Section \ref{sub:tcmc-sensitivity}. As a result we obtain differentially private means and covariances $\tilde{\vw}_{n_\vx}^\vx$, $\tilde{\mSigma}_{n_\vx}^\vx$ and $\tilde{\vw}_{n_\vy}^\vx$, $\tilde{\mSigma}_{n_\vx}^\vx$ at their respective users. All these quantities are then released to the tester whose role is to compute the test statistic and the corresponding p-value. In particular, the tester uses the statistic given by 
%\vspace{-2mm}
\begin{align}
 \tilde{s}_{n_\vx,n_\vy} =     \frac{n_\vx n_\vy}{n_\vx + n_\vy}(\tilde{\vw}_{n_\vx}^\vx-\tilde{\vw}_{n_\vy}^\vy)^\top (\tilde{\mSigma}_{n_\vx,n_\vy}+\gamma_n I)^{-1}(\tilde{\vw}_{n_\vx}^\vx-\tilde{\vw}_{n_\vy}^\vy),\nonumber 
\end{align}
where $\tilde{\mSigma}_{n_\vx,n_\vy}$ is the pooled covariance estimate, %\vspace{-2mm}
%\begin{equation}
   $ \tilde{\mSigma}_{n_\vx,n_\vy}=\frac{(n_\vx-1)\tilde{\mSigma}_{n_\vx}^\vx+(n_\vy-1)\tilde{\mSigma}_{n_\vy}^\vy}{n_\vx+n_\vy-2}.$
%\end{equation}

\section{Analysis of null distributions}\label{sec:null_analysis}
In the previous sections, we discussed necessary tools to make the kernel two sample tests private in two different settings by considering sensitivity analysis of quantities of interest.\footnote{Also look into the Appendix~\ref{app:sigam_half} and \ref{app:chi_square_noise} for other possible approaches.}. In this section, we consider the distributions of the test statistics under the null hypothesis $P=Q$ for each of the two privacy settings. 

%\subsection{Asymptotic null distributions}

\subsection{Trusted-curator setting: perturbed mean and covariance}
\label{sub:null_tcmc}
%In this section, we describe a way to achieve the privacy goal by perturbing mean and covariance separately. In this setting, we will have a privacy budget given which we can spend on making mean and covariance both private respectively. Since the noise added in the mean and covaraince matrix are independent hence final privacy can be said to be the sum of individual privacy level \anant{@mijung @dino check if it makes sense}. Previously in lemma~\ref{lem:dwork} and in theorems~\ref{thm:wishart1}, \ref{lem:wishart2} and \ref{thm:wishart3}, it has been described ways to achieve private mean and covariance respectively.
In this scheme, noise is added both to the mean vector $\vw_n$ and to the covariance matrix $\mSigma_n$ (by dividing the privacy budget between these two quantities). Let us denote the perturbed mean by $\tilde{\vw}_n$ and perturbed covariance with $\tilde{\mSigma}_n$. The noisy version of the test statistic $\tilde{s}_n $ is then given by
\begin{align}
    \tilde{s}_n =
    n\tilde{\vw}_n\trp ~(\tilde{\mSigma}_n+\gamma_n I)^{-1} \tilde{\vw}_n \label{eq:priv_test_stats}
\end{align}
where $\gamma_n$ is a regularization parameter just like in the non-private statistic \eqref{eq:teststatistic}. We show below that the asymptotic null distribution (as sample size $n\to\infty$) of this private test statistic is in fact identical to that of the non-private test statistic. Intuitively, this is to be expected: as the number of samples increases, the contribution to the aggregate statistics of any individual observation diminishes, and the variance of the added noise goes to zero.%We here provide a result which essentially says that if we add noise in the mean statstics as according to the respective sensitivity of mean and covariance, then as the number of samples gets larger the noisy statistics converges to the true one. This is also aligned with the true spirit of differential privacy as the number of samples increases, the importance of a single sample goes down and the computed function is already private.
\begin{theorem} \label{lem:sltusky}
Assuming the Gausssian noise for $\tilde{\vw}_n$ with the sensitivity bound in \eqref{eq:mean_sensitivity} and the perturbation mechanism introduced in Section \ref{sub:tcmc-sensitivity} for $\tilde{\mSigma}_n$, $\tilde{s}_n$ and $s_n$ converge to the same limit in distribution, as $n\to\infty$. %\anant{@dino shall we make this statement a bit more formal ?}
\end{theorem}
Proof is provided in Appendix \ref{sec:proof_theorem_5_1}.
%
%Since the true test statistics $s_n$ is a.s. asymptotically distributed as a $\chi^2-$random variable~\cite{chwialkowski2015fast}, hence using the lemma~\ref{lem:sltusky}, we know that the private statistics $\tilde{s}_n$ will also be asymptotically distributed as a $\chi^2-$random variable. 
%
Based on the Theorem,  it is tempting to ignore the additive noise and rely on the asymptotic null distribution.
%to compute the test threshold.
%consider a testing procedure which relies on the asymptotic null distribution and computes the test threshold in exactly the same way (essentially ignoring the "privatizing" noise). 
However, as demonstrated in  \secref{experiments}, such tests have an inflated number of false positives.
%, and hence is not suitable for testing. 
We propose a non-asymptotic regime in order to improve  approximations  of  the  null  distribution when computing the test threshold. 
%or  suffer  from  poor  power  performance.  
%
In particular, recall that we previously relied on $\sqrt n \vw_n$ converging to a zero-mean multivariate normal distribution $\mathcal N(0,\mSigma)$, with $\mSigma =\mathbb E[\vz \vz^\top]$~\cite{chwialkowski2015fast}. In the private setting, we will also approximate the distribution of $\sqrt n \tilde{\vw}_n$ with a multivariate normal, but consider explicit non-asymptotic covariances which appear in the test statistic. Namely, the covariance of $\sqrt n \tilde{\vw}_n$ is $\mSigma + n\sigma_{\vn}^2 I$ and its mean is 0, so we will approximate its distribution by $\mathcal N(0, \mSigma + n\sigma_{\vn}^2 I)$. The test statistic can be understood as a squared norm of the vector $\sqrt n \left(\tilde{\mSigma}_n+\gamma_n I\right)^{-1/2}\tilde{\vw}_n$. Under the normal approximation to $\sqrt n \tilde{\vw}_n$ and by treating $\tilde{\mSigma}_n$ as fixed (note that this is a quantity released to the tester), $\sqrt n \left(\tilde{\mSigma}_n+\gamma_n I\right)^{-1/2}\tilde{\vw}_n$ is another multivariate normal, i.e. $\mathcal N(0, \mC)$, where  
%\vspace{-2mm}
%\begin{align*}
    $\mC=(\tilde{\mSigma}_n+\gamma_n I)^{-1/2}(\mSigma + n\sigma_{\vn}^2 I)(\tilde{\mSigma}_n+\gamma_n I)^{-1/2}.$
%\end{align*}
 The overall statistic thus follows a distribution given by a weighted sum $\sum_{j=1}^J \lambda_j \chi^2_j$ of independent chi-squared distributed random variables, with the weights $\lambda_j$ given by the eigenvalues of $\mC$. Note that this approximation to the null distribution depends on a non-private true covariance $\mSigma$. While that is clearly not available to the tester, we propose to simply replace this quantity with the privatized empirical covariance, i.e. $\tilde{\mSigma}_n$, so that the tester approximates the null distribution with $\sum_{j=1}^J \tilde\lambda_j \chi^2_j$, where $\tilde\lambda_j$ are the eigenvalues of 
 %\vspace{-3mm}
%\begin{align*}
    $\tilde \mC=(\tilde{\mSigma}_n+\gamma_n I)^{-1}(\tilde{\mSigma}_n + n\sigma_{\vn}^2 I),$
%\end{align*}
i.e. $\tilde \lambda_j = \frac{\tau_j +n\sigma_{\vn}^2}{\tau_j+\gamma_n}$, where $\{\tau_j\}$ are the eigenvalues of $\tilde{\mSigma}_n$ (note that $\tilde \lambda_j\to 1$ as $n\to\infty$ recovering back the asymptotic null).
This approach, while a heuristic, gives a correct Type I control and good power performance, unlike the approach which relies on the asymptotic null distribution and ignores the presence of privatizing noise.

\subsection{Trusted-curator setting: perturbed test statistic}

In this section, we will consider how directly perturbing the test statistic impacts the null distribution. %However, this task is not simple due to various reasons which we will describe further in this section. 
To achieve private test statistics, we showed that we can simply use add Gaussian noise\footnote{While this may produce negative privatized test statistics, which may at first appears problematic, this poses no issues for performing the actual test. Indeed, the test threshold is appropriately adjusted to take into account that the distribution of the test statistic can take negative values.% This is similar to performing a standard kernel two-sample test with unbiased MMD estimators \cite{Gretton2012jmlr}. %While MMD is always a non-negative quantity, its unbiased estimator can take negative values, but this does not affect the correctness of the test.
See Appendix~\ref{app:chi_square_noise} and \ref{app:sigam_half} for alternative approaches for privatizing the test statistic.
%, we discuss in detail the difficulties in having alternate approaches to achieve privacy goals for kernel two sample testing in the setting of perturbed statistics. 
} using the analytic Gaussian mechanism, described in Section \ref{sec:test_statistic}.

%
%noise onto the test statistics $s_n$.
Similarly to Theorem \ref{lem:sltusky}, we have a similar theorem below, which says that the perturbed statistic then has the same asymptotic null distribution as the original statistic.
%Using the data independent upper bound on the  sensitivity $GS_2(s_n)$ from  Theorem~\ref{thm:test_sn}, we write $\tilde{s}_n = s_n + \eta$ where $\eta \sim \mathcal{N}(0, \sigma_\eta^2(\epsilon,\delta,n))$ and $\sigma_\eta(\epsilon, \delta,n) \geq \frac{GS_2(s_n)\sqrt{2 \log{(1.25/\delta_1)}}}{\epsilon_1}$. The perturbed statistic then has the same asymptotic null distribution as the original statistic.
\begin{theorem}
Using the noise variance $\sigma_\eta^2(\epsilon,\delta,n)$ defined by the upper bound in Theorem~\ref{thm:test_sn}, $\tilde{s}_n$ and $s_n$ converge to the same limit in distribution, as $n\to\infty$.
\end{theorem}
%\vspace{-2mm}
%\begin{proof}
The proof follows immediately from $\sigma_\eta(\epsilon, \delta,n) \rightarrow 0$, as $n\to\infty$.
%The only important thing to note here is that as number of samples $n$ grows the variance $\sigma(\epsilon, \delta) \rightarrow 0$. Also since $\mathbb{E}~\eta = 0$, the ditribution of noise will converge to point mass at $0$ as $n$ grows large.  Again, by using slutsky's theorem it is easier to see that the perturbed test statistics would converge to the true test statistics $s_n$ in distribution which is essentially $\chi^2-$distribution with $J$ degree of freedom.
%\end{proof}
%\vspace{-2mm}
%Just like in the case of perturbed mean and covariance, resorting to the asymptotic null distribution and ignoring the noise added to the test statistic, results in tests with incorrect Type I control. Thus, 
As in the case of perturbed mean and covariance, we consider approximating the null distribution with the sum of the chi-squared with $J$ degrees of freedom and a normal $\mathcal{N}(0, \sigma_\eta^2(\epsilon,\delta,n))$, i.e., the distribution of the true statistic is approximated with its asymptotic version, whereas we use exact non-asymptotic distribution of the added noise. The test threshold can then easily be computed by a Monte Carlo test which repeatedly simulates the sum of these two random variables. It is important to
note that since $\sigma_\eta^2(\epsilon,\delta,n)$ is \textit{independent of the data} (as shown in Appendix \ref{app:sens_stats}), an untrusted tester can simulate the approximate null distribution without compromising privacy. %\leon{I added the punchline that was missing, someone check my wording.}

\subsection{No-trusted-entity setting}
Similarly as in section \ref{sub:null_tcmc}, as $n_\vx,n_\vy\to \infty$ such that $n_\vx/n_\vy \to \rho\in(0,1)$, asymptotic null distribution of this test statistic remains unchanged as in the non-private setting, i.e. it is the chi-squared distribution with $J$ degrees of freedom. However, by again considering the non-asymptotic case and applying a chi-squared approximation, we get improved power and type I control. In particular, the test statistic is close to a weighted sum $\sum_{j=1}^J \lambda_j \chi^2_j$ of independent chi-square distributed random variables, with the weights $\lambda_j$ given by the eigenvalues of
$\mC=\frac{n_\vx n_\vy}{n_\vx + n_\vy}(\tilde{\mSigma}_{n_\vx,n_\vy}+\gamma_n I)^{-1/2}(\mSigma^\vx/n_\vx+\mSigma^\vy/n_\vy+(\sigma_{\vn^\vx}^2+\sigma_{\vn^\vy}^2)I)(\tilde{\mSigma}_{n_\vx,n_\vy}+\gamma_n I)^{-1/2}$,
where $\mSigma^\vx$ and $\mSigma^\vy$ are the true covariances within each of the samples, $\sigma_{\vn^\vx}^2$ and $\sigma_{\vn^\vy}^2$ are the variances of the noise added to the mean vectors $\vw_{n_\vx}$ and $\vw_{n_\vy}$, respectively.
While $\mSigma^\vx$ and $\mSigma^\vy$ are clearly not available to the tester, the tester can replace them with their privatized empirical versions $\tilde{\mSigma}_{n_\vx}^\vx$ and $\tilde{\mSigma}_{n_\vy}^\vy$ and compute eigenvalues  $\tilde{\lambda}_j$ of 
$\tilde \mC=\frac{n_\vx n_\vy}{n_\vx + n_\vy}(\tilde{\mSigma}_{n_\vx,n_\vy}+\gamma_n I)^{-1/2}(\tilde{\mSigma}_{n_\vx}^\vx/n_\vx+\tilde{\mSigma}_{n_\vy}^\vy/n_\vy+(\sigma_{\vn^\vx}^2+\sigma_{\vn^\vy}^2)I)(\tilde{\mSigma}_{n_\vx,n_\vy}+\gamma_n I)^{-1/2}$.
Similarly as in the trusted-curator setting, we demonstrate that this corrected approximation to the null distribution leads to significant improvements in power and Type I control.

\section{Experiments}\label{sec:experiments}
\begin{figure*}
\centering
\includegraphics[scale=0.33]{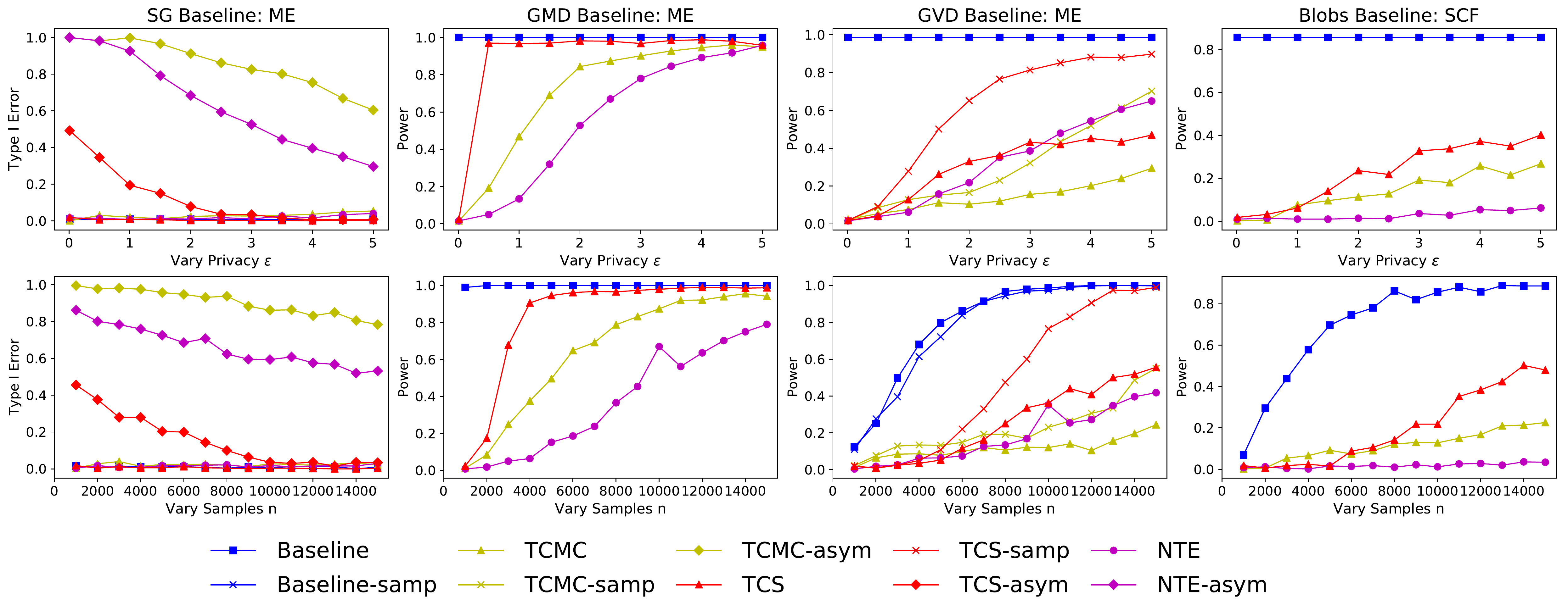}
\caption{Type I error for the SG dataset , Power for the GMD, GVD, Blobs dataset over $500$ runs, with $\delta=1e^-5$. 
    \textbf{Top}: Varying $\epsilon$ with $n=10000$. \textbf{Bottom}: Varying $n$ with $\epsilon=2.5$. Here *-asym represents using the asymptotic $\chi^2$ null distribution, while *-samp represents sampling locations and using the median heuristic bandwidth.} %Here, we do not display the power analysis for asymptotic regimes, due to the large inflation of Type I error.}
\label{fig:vary_main}
\end{figure*}

Here we demonstrate the effectiveness of our private kernel two-sample test on both synthetic and real problems, for testing $H_0: P = Q$. 
%
%
%As introduced, we consider two settings in which different quantities are noised-up. For clarity,  we introduce a few acronyms for each case. We denote the trusted-curator setting by TC. Under TC, we denote the case where we perturb the statistic by TCS; and the case where we perturb the mean and covariance by TCMC. We denote the no-trusted-entity setting by NTE. 
%
%
The total sample size is denoted by $N$ and the number of test set samples by $n$. We set the significance level to $\alpha = 0.01$. Unless specified otherwise use the isotropic Gaussian kernel with a bandwidth $\theta$ and fix the number of test locations to $J=5$. Under the trusted-curator (TC) setting, we use $20\%$ of the samples $N$ as an independent training set to optimize the test locations and $\theta$ using gradient descent as in \cite{Jitkrittum2016}. 
%we can no longer do this, as neither side has access to both set of samples. 
Under the no-trusted-entity  (NTE) setting, we randomly sample $J$ locations and calculate the median heuristic bandwidth \cite{gretton2012optimal} from the training set. %In case, no public training set is available, each side of data owners can release a private median heuristic bandwidth and an average of them can be used instead.
 %\footnote{See Appendix \ref{sec:optimizing_test_location_appendix} for a sketch of an algorithm that enables two data owners to jointly optimize test locations. }
%
%it is simply the square root of half of median square distance 
\\ \\
For all our experiments, we average them over $500$ runs, where each run repeats the simulation or randomly samples without replacement from the data set.
%, note that on each run, we also use a different set of noise samples to make our tests differential private. 
We then report the empirical estimate of $\mathcal{P}(\tilde{s}_n > T_\alpha)$, computed by proportion of times the statistic $\tilde{s}_n$ is greater than the $T_\alpha$, where $T_\alpha$ is the test threshold provided by the corresponding approximation to the null distribution.
We fix the regularization parameter $\gamma$ to $0.001$ for the TC under perturbed test statistics (TCS). In TCMC and NTE, given the privacy budget of $(\epsilon, \delta)$, we use $(0.5 \epsilon, 0.5 \delta)$ to perturb the mean and covariance seperately. We compare these to its non-private counterpart ME and SCF. More experimental details and experiments can be found in Appendix \ref{sec:experiments_appendix}.
\subsection{Synthetic data}
We demonstrate our tests on $4$ separate synthetic problems, namely, Same Gaussian (SG), Gaussian mean difference (GMD), Gaussian variance difference (GVD) and Blobs, with the specifications of $P$ and $Q$ summarized in Table. 2. The same experimental setup was used in \cite{Jitkrittum2016}. For the Blobs dataset, we use the SCF approach as the baseline, and also the basis for our algorithms, since \cite{chwialkowski2015fast, Jitkrittum2016} showed that SCF outperforms the ME test here.%, as it considers differences in the spectral domain, and this is where the differences lies in the Blobs dataset. %We now perform two experiments, 
\paragraph{Varying privacy level $\epsilon$}
We now fix the test sample size $n$ to be $10 \ 000$, and vary $\epsilon$ between $0$ and $5$ with a fixed $\delta=1e-5$. The results are shown in the top row of Figure \ref{fig:vary_main}. For SG dataset, where $H_0: P = Q$ is true, we can see that if one simply applies the asymptotic null distribution of a $\chi^2$ on top, we will obtain a massively inflated type I error. This is however not the case for TCMC, TCS and NTE, where the type I error is approximately controlled at the right level, this is shown more clearly in Figure \ref{fig:type_1} in the Appendix. In GMD, GVD and Blobs dataset, the null hypothesis does not hold, and we see that our algorithms indeed discover this difference. As expected we observe a trade-off between privacy level and also power, for increasing privacy (decreasing $\epsilon$), we have less power. These experiments also reveals the order of performance of these algorithms, i.e. TCS $>$ TCMC $>$ NTE. This is not surprising, as for TCMC and NTE, we are pertubing the mean and covariance separately, rather than the statistic directly, which is the direct quantity we want to protect. %However, the released means and covariances from TCMC and NTE can have potentially wider applications in general. %Although, NTE performs the worse%\footnote{Alternatively, in NTE setting one could allocate the privacy budget into four to perturb two means and two covariances, such that the total privacy loss at the tester is the same as that in the case of TCMC. This means, we now protect an individual's privacy as a pair (one in $\Dat_x$ and another in $\Dat_y$).} 
%here, it is worth mentioning that the setting it considers is considerably harder, as we perturb the mean and covariance from each party, this is unlike in TCMC, where this pertubation for mean and covariance is only done once.

The power analysis for the SVD and Blobs dataset also reveal the interesting nature of sampling versus optimisation in our two settings. In the SVD dataset, we observe that NTE performs better than TCS and TCMC, however if we use the same test locations and bandwidth of NTE for TCS and TCMC, the order of performance is as we expect, better for sampling over optimization. However, in the Blobs dataset, we observe that NTE has little or no power, because this dataset is sensitive to the choice of test frequency locations, highlighting the importance of optimisation in this case.%, however in practice this is unavailable for NTE.  %\leon{can someone fix, doesnt sound right...}
\paragraph{Varying test sample size $n$}
We now fix $\epsilon=2.5$, $\delta=1.0^{-5}$ and vary $n$ from $1000$ to $15\ 000$. The results are shown in the bottom row of Figure \ref{fig:vary_main}. The results for the SG dataset further reinforce the importance of not simply using the asymptotic null distribution, as even at very large sample size, the type I error is still inflated when naively computing the test threshold form a chi-squared distribution. This is not the case for TCMC, TCS and NTE, where the type I error is approximately controlled at the correct level for all sample sizes, as shown in Figure \ref{fig:type_1}. %For the GMD, GVD and Blobs dataset, we observe increasing power for increasing sample size (Isn't it obvious).%, as which is due to two factors. One factor is the standard better estimation of $\mSigma_n$ and $\vw_n$, the other factor is that the level of noise needed to make it differential private reduces for increasing sample size, as an individual contribution diminishes.
\subsection{Real data: Celebrity age data}
We now demonstrate our tests on a real life celebrity age dataset, namely the IMDb-WIKI dataset \cite{rothe2016deep}, containing $397\ 949$ images of $19 \ 545$ celebrities and their corresponding age labels. %This dataset was constructed by crawling from the IMDb and the Wikipedia website, with potentially multiple images of a celebrity over time. Previously,  \cite{rothe2016deep,law2017bayesian} have considered the problem of age prediction, making use of a convolutional neural network (CNN) with a VGG-16 architecture. 
Here, we will follow the preprocessing of \cite{law2017bayesian}, %where images from the same celebrity are placed into the same bag, and the bag label is calculated as the mean age of that celebrity's images. We now 
and use this to construct two datasets, under25 and 25to35. Here the under25 dataset is the images where the corresponding celebrity's bag label is  $< 25$, and the 25to35 dataset is the images corresponding to the celebrity's bag label that is between 25 and 35. The dataset under25 contains $58 095$ images, and the dataset 25to35 contains $126415$ images.
\begin{figure*}
\begin{minipage}{0.48\linewidth}
\centering
\hfill
  \includegraphics[scale=0.35]{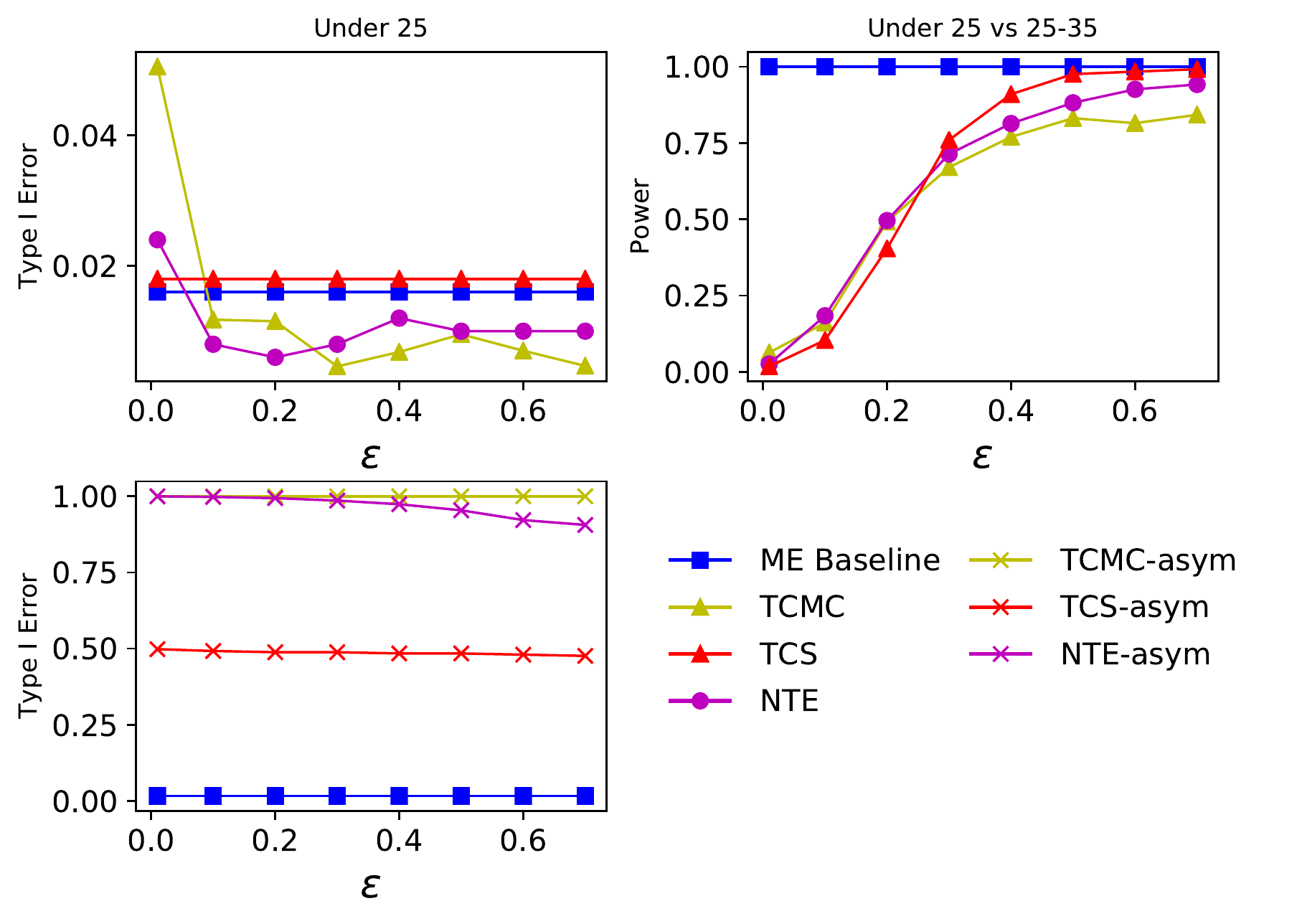}
\hfill\null

\caption{Type I error for the under25 only test, Power for the under25 vs 25to35 test over $500$ runs, with $n=2500,\delta=1e^{-5}$. *-asym represents using the asymptotic $\chi^2$ null distribution. }
\label{fig:faces}
\end{minipage}
~\begin{minipage}{.48\linewidth}
\small
\captionof{table}{Synthetic problems (Null hypothesis $H_0$  holds only for SG). Gaussian Mixtures in $\mathbb{R}^2$, also studied in \cite{Jitkrittum2016,chwialkowski2015fast,GreSriSejStrBalPonFuk2012}.}
\FloatBarrier

\label{my-label}
\begin{tabular}{lll}
\textbf{Data} & $\mathbf{P}$                     & $\mathbf{Q}$\\
\hline
SG   & $\mathcal{N}(0, I_{50})$ & $\mathcal{N}(0, I_{50})$                   \\
GMD  & $\mathcal{N}(0, I_{100})$ & $\mathcal{N}((1,0,\dots, 0)^\top, I_{100})$ \\
GVD  & $\mathcal{N}(0, I_{50})$ & $\mathcal{N}(0, diag(2,1,\dots,1))$\\
\end{tabular}
\FloatBarrier
\centering
\vspace{0.2cm}
\includegraphics[scale=0.35]%scale=0.
%Interpretable distribution392features with maximum testing powe35 
{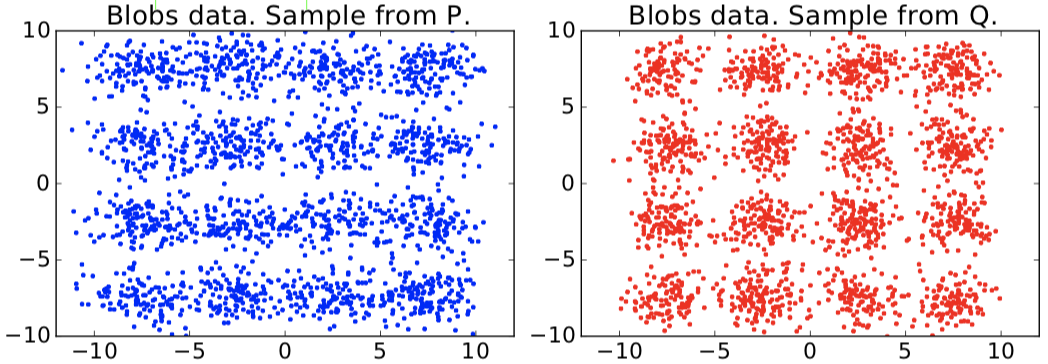}
\captionof{figure}{Blobs data sampled from $\mathbf{P}$ on the left and blobs data sampled from $\mathbf{Q}$ in the right. }
\end{minipage}
\end{figure*}

For this experiment, we will focus on using the ME version of the test  and consider the kernel
$$
k(\mathbf{x},\mathbf{y}) = \exp \left( -\frac{||\varphi(\mathbf{x})-\varphi(\mathbf{y})||^2}{2\theta^2} \right)
$$
where $\varphi(x): \mathbb{R}^{256 \times 256} \rightarrow \mathbb{R}^{4096}$ is the feature map learnt by the CNN in \cite{rothe2016deep}, mapping the image in the original pixel space to the last layer. %By performing a differential private version of the two sample test, we are protecting the images in the pixel space, which in practice could contain sensitive information. 
For our experiment, we take $N=3125$, and use $20\%$ of the data for sampling test locations, and calculation of the median heuristic bandwidth. Note here we do not perform optimization, due to the large dimension of the feature map $\varphi$. We now perform two tests, for one test we compare samples from under25 only (i.e. $H_0: P=Q$ holds), and the other we compares samples from under25 to samples from 25to35 (i.e. $H_0: P=Q$ does not hold). The results are shown in Figure \ref{fig:faces} for $\epsilon$ from $0.1$ to $0.7$. We observe that in the under25 only test, the TCMC, TCS and NTE all achieve the correct Type I error rate, this is unlike their counterpart that uses the $\chi^2$ asymptotic null distribution. In the under25 vs 25to35 two sample test, we see that our algorithms achieve maximal power at a high level of privacy, protecting the original images from malicious intent.
\section{CONCLUSION}
While kernel-based hypothesis testing provides flexible statistical tools for data analysis, its utility in differentially private settings is not well understood. 
%This is due to complex test statistics and their null distributions which are typically data dependent. 
We investigated differentially private kernel-based two-sample testing procedures, by making use of the sensitivity bounds on the quantities used in the test statistics. While asymptotic null distributions for the modified procedures remain unchanged, ignoring additive noise can lead to an inflated number of false positives. Thus, we propose new approximations of the null distributions under the private regime which give correct Type I control and good power-privacy tradeoffs, as demonstrated in extensive numerical evaluations.   
\section*{Acknowledgements}
We would like to thank Wittawat Jitkrittum for useful discussions and providing code for which our experiments are based on. HCLL is supported by the EPSRC and MRC through the OxWaSP CDT programme (EP/L016710/1). DS is supported in part by the ERC (FP7/617071) and by The Alan Turing Institute (EP/N510129/1). AR and MP thank the generous MPI society for their support. MP would like to also thank University of Tuebingen for their generous funding.

{\footnotesize
\bibliography{private-kl}

\begin{thebibliography}{10}

\bibitem{balle2018improving}
Borja Balle and Yu-Xiang Wang.
\newblock Improving the gaussian mechanism for differential privacy: Analytical
  calibration and optimal denoising.
\newblock 2018.

\bibitem{balog2017privacy}
Matej Balog, Ilya Tolstikhin, and Bernhard Sch\"olkopf.
\newblock Differentially {Private} {Database} {Release} via {Kernel} {Mean}
  {Embeddings}.
\newblock {\em arXiv:1710.01641 [stat]}, October 2017.
\newblock arXiv: 1710.01641.

\bibitem{Borgwardt2006}
Karsten~M. Borgwardt, Arthur Gretton, Malte~J. Rasch, Hans-Peter Kriegel,
  Bernhard Sch\"{o}lkopf, and Alex~J. Smola.
\newblock Integrating structured biological data by kernel maximum mean
  discrepancy.
\newblock {\em Bioinformatics}, 22(14):e49--e57, July 2006.

\bibitem{chaudhuri2011differentially}
Kamalika Chaudhuri, Claire Monteleoni, and Anand~D Sarwate.
\newblock Differentially private empirical risk minimization.
\newblock {\em Journal of Machine Learning Research}, 12(Mar):1069--1109, 2011.

\bibitem{chwialkowski2015fast}
Kacper~P Chwialkowski, Aaditya Ramdas, Dino Sejdinovic, and Arthur Gretton.
\newblock Fast two-sample testing with analytic representations of probability
  measures.
\newblock In {\em NIPS}, pages 1981--1989, 2015.

\bibitem{dwork2006our}
Cynthia Dwork, Krishnaram Kenthapadi, Frank McSherry, Ilya Mironov, and Moni
  Naor.
\newblock Our data, ourselves: Privacy via distributed noise generation.
\newblock In {\em Eurocrypt}, volume 4004, pages 486--503. Springer, 2006.

\bibitem{dwork2006calibrating}
Cynthia Dwork, Frank McSherry, Kobbi Nissim, and Adam Smith.
\newblock Calibrating noise to sensitivity in private data analysis.
\newblock In {\em TCC}, volume 3876, pages 265--284. Springer, 2006.

\bibitem{Dwork14}
Cynthia Dwork and Aaron Roth.
\newblock The algorithmic foundations of differential privacy.
\newblock {\em Found. Trends Theor. Comput. Sci.}, 9:211--407, August 2014.

\bibitem{DworkTT014}
Cynthia Dwork, Kunal Talwar, Abhradeep Thakurta, and Li~Zhang.
\newblock Analyze {G}auss: optimal bounds for privacy-preserving principal
  component analysis.
\newblock In {\em Symposium on Theory of Computing, {STOC} 2014, New York, NY,
  USA, May 31 - June 03, 2014}, pages 11--20, 2014.

\bibitem{FlaSejCunFil2016}
S.~Flaxman, D.~Sejdinovic, J.P. Cunningham, and S.~Filippi.
\newblock {{Bayesian Learning of Kernel Embeddings}}.
\newblock In {\em Uncertainty in Artificial Intelligence (UAI)}, pages
  182--191, 2016.

\bibitem{Gaboardi2016}
Marco Gaboardi, Hyun~Woo Lim, Ryan Rogers, and Salil~P. Vadhan.
\newblock Differentially private chi-squared hypothesis testing: Goodness of
  fit and independence testing.
\newblock In {\em Proceedings of the 33rd International Conference on
  International Conference on Machine Learning - Volume 48}, ICML'16, pages
  2111--2120, 2016.

\bibitem{Gaboardi17}
Marco Gaboardi and Ryan~M. Rogers.
\newblock Local private hypothesis testing: Chi-square tests.
\newblock {\em CoRR}, abs/1709.07155, 2017.

\bibitem{Gretton2006}
Arthur Gretton, Karsten~M. Borgwardt, Malte Rasch, Bernhard Sch\"{o}lkopf, and
  Alex~J. Smola.
\newblock A kernel method for the two-sample-problem.
\newblock In B.~Sch\"{o}lkopf, J.~C. Platt, and T.~Hoffman, editors, {\em
  NIPS}, pages 513--520. MIT Press, 2007.

\bibitem{Gretton2012jmlr}
Arthur Gretton, Karsten~M. Borgwardt, Malte~J. Rasch, Bernhard Sch\"{o}lkopf,
  and Alexander Smola.
\newblock A kernel two-sample test.
\newblock {\em J. Mach. Learn. Res.}, 13(1):723--773, March 2012.

\bibitem{Gretton2009_spectral}
Arthur Gretton, Kenji Fukumizu, Za\"{\i}d Harchaoui, and Bharath~K.
  Sriperumbudur.
\newblock A fast, consistent kernel two-sample test.
\newblock In {\em NIPS}, pages 673--681. 2009.

\bibitem{gretton2012optimal}
Arthur Gretton, Dino Sejdinovic, Heiko Strathmann, Sivaraman Balakrishnan,
  Massimiliano Pontil, Kenji Fukumizu, and Bharath~K Sriperumbudur.
\newblock Optimal kernel choice for large-scale two-sample tests.
\newblock In {\em NIPS}, pages 1205--1213, 2012.

\bibitem{GreSriSejStrBalPonFuk2012}
Arthur Gretton, Bharath~K. Sriperumbudur, Dino Sejdinovic, Heiko Strathmann,
  Sivaraman Balakrishnan, Massimiliano Pontil, and Kenji Fukumizu.
\newblock Optimal kernel choice for large-scale two-sample tests.
\newblock In {\em NIPS}, volume~25, pages 1205--1213. 2012.

\bibitem{hall2013differential}
Rob Hall, Alessandro Rinaldo, and Larry Wasserman.
\newblock Differential privacy for functions and functional data.
\newblock {\em Journal of Machine Learning Research}, 14(Feb):703--727, 2013.

\bibitem{Homer2008}
Nils Homer, Szabolcs Szelinger, Margot Redman, David Duggan, Waibhav Tembe,
  Jill Muehling, John~V. Pearson, Dietrich~A. Stephan, Stanley~F. Nelson, and
  David~W. Craig.
\newblock Resolving individuals contributing trace amounts of dna to highly
  complex mixtures using high-density snp genotyping microarrays.
\newblock {\em PLOS Genetics}, 4(8):1--9, 08 2008.

\bibitem{Jain13}
Prateek Jain and Abhradeep Thakurta.
\newblock Differentially private learning with kernels.
\newblock In {\em Proceedings of the 30th International Conference on Machine
  Learning, ICML 2013, Atlanta, GA, USA, 16-21 June 2013}, pages 118--126, July
  2013.

\bibitem{Jitkrittum2016}
Wittawat Jitkrittum, Zolt\'{a}n Szab\'{o}, Kacper Chwialkowski, and Arthur
  Gretton.
\newblock Interpretable distribution features with maximum testing power.
\newblock In {\em NIPS}, 2016.

\bibitem{Johnson2013}
Aaron Johnson and Vitaly Shmatikov.
\newblock Privacy-preserving data exploration in genome-wide association
  studies.
\newblock In {\em Proceedings of the 19th ACM SIGKDD International Conference
  on Knowledge Discovery and Data Mining}, KDD '13, pages 1079--1087, New York,
  NY, USA, 2013. ACM.

\bibitem{law2017bayesian}
Ho~Chung~Leon Law, Dougal~J Sutherland, Dino Sejdinovic, and Seth Flaxman.
\newblock Bayesian approaches to distribution regression.
\newblock In {\em Uncertainty in Artificial Intelligence (UAI)}, 2017.

\bibitem{Muandet2017}
Krikamol Muandet, Kenji Fukumizu, Bharath Sriperumbudur, and Bernhard
  Schölkopf.
\newblock Kernel mean embedding of distributions: A review and beyond.
\newblock {\em Foundations and Trends® in Machine Learning}, 10(1-2):1--141,
  2017.

\bibitem{rahimi2008random}
Ali Rahimi and Benjamin Recht.
\newblock Random features for large-scale kernel machines.
\newblock In {\em Advances in neural information processing systems}, pages
  1177--1184, 2008.

\bibitem{rogers2017new}
Ryan Rogers and Daniel Kifer.
\newblock A new class of private chi-square hypothesis tests.
\newblock In {\em Artificial Intelligence and Statistics}, pages 991--1000,
  2017.

\bibitem{rothe2016deep}
Rasmus Rothe, Radu Timofte, and Luc Van~Gool.
\newblock Deep expectation of real and apparent age from a single image without
  facial landmarks.
\newblock {\em International Journal of Computer Vision}, pages 1--14, 2018.

\bibitem{Wahba90a}
G.~Wahba.
\newblock {\em Spline Models for Observational Data}.
\newblock Society for Industrial and Applied Mathematics, 1990.

\end{thebibliography}
}
\bibliographystyle{plain}

\newpage
\onecolumn
\appendix
\begin{center}
{\centering \LARGE Appendix }
\vspace{1cm}
\sloppy

\end{center}

\section{Covariance Perturbation}\label{app:cov_perturb}

\label{thm:wishart3}

\begin{theorem}
[Modified Analyze Gauss]
\label{thm:anaylzegauss}
%Given a 2nd-moment matrix $\mLambda \in \mathbb{R}^{J \times J}$, 
Draw Gaussian random variables $\veta \sim \Nrm(0, \beta^2 \mI_{J(J+1)/2})$ where $\beta = \frac{\kappa^2 J\sqrt{2 \log(1.25/\delta_2)}}{(n-1)\epsilon_2}$. 
Using $\veta$, we construct a upper triangular
matrix (including diagonal), then copy the upper
part to the lower part so that the resulting matrix $\mD$ becomes symmetric.
The perturbed matrix $\tilde\mLambda = \mLambda + \mD$ is ($\epsilon_2, \delta_2$)-differentially private\footnote{
To ensure $\tilde{\mLambda}$ to be positive semi-definite, we project any  negative sigular values to a small positive value (e.g., 0.01).}.  
\end{theorem}
The proof is the same as the proof for Algorithm 1 in \cite{DworkTT014} with the exception that the global sensitivity of $\mLambda$ is
%$\kappa^2 J / n$.  
%
%Another way to obtain a $(\epsilon_2, \delta_2)$-DP approximation to the positive semi-definite matrix $\mLambda$ using the Analyze Gauss method is given below. 
%We first draw Gaussian random variables $\veta \sim \Nrm(0, \beta \mI_{J(J+1)/2})$ where $\beta = \frac{2 \log(1.25/\delta_2)GS^2(\mLambda)}{\epsilon_2^2}$ and $GS(\mLambda)$ is the global sensitivity of $\mLambda$. In our case, $GS(\mLambda)=2B^2$, since
\begin{align}
GS(\mLambda) = \max_{\Dat, \Dat'}\| \mLambda(\Dat) - \mLambda(\Dat')\|_F = \max_{\vv, \vv'} \| \vv \vv{\trp} - \vv' {\vv'}{\trp}\|_F
\leq \tfrac{\kappa^2 J}{n-1}, \label{eq:privtae_psd_method3}
\end{align} where $\vv$ is the single entry differing in $\Dat$ and ${\Dat}'$, and $\|\vv\|_2 \leq \frac{\kappa \sqrt{J}}{\sqrt{n-1}}$. 
%Here we can think of it as an $n^2$ dimensional vector.
%Using $\veta$, we construct a upper triangular
%matrix (including diagonal), then copy the upper
%part to the lower part so that the resulting matrix becomes symmetric. 

\section{Sensitivity of $\vw_n^\top \left(\mSigma_n+\gamma_n\mathbf{I}\right)^{-1}\vw_n $}\label{app:sens_stats}
We first introduce a few notations, which we will use for the sensitivity analysis.
\begin{itemize}
\item We split $\vw_n = \vm + \frac{1}{\sqrt{n}}\vv $, where $\vm= \frac{1}{n}\sum_{i=1}^{n-1} \vz_i$ and $\vv = \frac{1}{\sqrt{n}} \vz_n$.
\item Similarly, we split $\mLambda = \mM{\trp}\mM + \frac{n}{n-1}\vv \vv{\trp} + \gamma_n \mI$, where ${\mM}{\trp}\mM = \frac{1}{n-1}\sum_{i=1}^{n-1} \vz_i\vz_i{\trp}$, we denote $\mM_{\gamma_n} = {\mM}^\top \mM + \gamma_n \mI $, where $\gamma_n > 0$
\item We put a dash for the quantities run on the neighbouring dataset ${\Dat}'$, i.e., the mean vector is ${\vw_n}'$, the 2nd-moment matrix is ${\mLambda}'$ (including a regularization term of $\gamma_n \mI$). Here, ${\vw_n} = \vm + \frac{1}{\sqrt{n}}\vv'$, $\vv' = \frac{1}{\sqrt{n}} \vz_n'$, and $\mLambda' = \mM{\trp}\mM + \frac{n}{n-1}\vv' {\vv'}{\trp} + \gamma_n \mI = \mM_{\gamma_n} + \frac{n}{n-1}\vv' {\vv'}{\trp}$. Similarly, the covariance given the dataset $\Dat$ is $\mSigma = \mLambda - \frac{n}{n-1}\vw_n \vw_n{\trp}$ and the covariance given the dataset $\Dat'$ is $\mSigma' = \mLambda' - \frac{n}{n-1}{\vw_n} {{\vw_n}}{\trp}$. 
\item Note that $\mLambda$ and $\mM_{\gamma_n}$ is positive definite, and hence invertible and have positive eigenvalues, we let eigen-vectors of $\mM_{\gamma_n}$ are denoted by $\vu_1, \cdots, \vu_J$ and the corresponding eigenvalues by $\mu_1, \cdots, \mu_J$. We also define the eigen-vectors such that $\mQ$ is orthogonal. Here $\mQ$ has columns given by the eigen-vectors.
\end{itemize} 
The L2-sensitivity of test statistic is derived using a few inequalities that are listed below: 
\begin{align}
GS_2(s_n) &= \max_{\Dat, \Dat'} \quad \left|s_n (\Dat) - s_n(\Dat')\right|, \\
&= n \; \max_{\vv, \vv'} \left|\vw_n{\trp}\mSigma^{-1}\vw_n - {\vw_n}{\trp}{\mSigma'}^{-1}{\vw_n} \right| \\
&= n \; \max_{\vv, \vv'} \left|\vw_n{\trp}(\mLambda - \frac{n}{n-1}\vw_n {\vw_n{\trp}})^{-1}\vw_n - {\vw_n}{\trp}(\mLambda'-\frac{n}{n-1}{\vw_n}{{\vw_n}{\trp}})^{-1}{\vw_n} \right|, \\
&\leq 2n\;  \max_{\vv, \vv'}  \left|\vw_n{\trp}\mLambda^{-1}\vw_n - {\vw_n}{\trp}{\mLambda'}^{-1}{\vw_n} \right|, \mbox{ due to inequality I} \\
&\leq 2n \; \max_{\vv, \vv'} \left(\left|{{\vw'_n}}{\trp}(\mLambda^{-1}-{\mLambda'}^{-1}){{\vw'_n}}\right|+\big| \vw_n^\top \mLambda^{-1} \vw_n  - {{\vw'_n}}^\top \mLambda^{-1} {{\vw'_n}}\big| \right), \\
&\leq 2n \; \max_{\vv, \vv'} \| \vw_n \|_2^2 \| \mLambda^{-1} - {\mLambda'}^{-1} \|_F + \frac{4\kappa^2 J}{n} \frac{\sqrt{J}}{\mu_{min}(\mLambda)},  \mbox{Cauchy Schwarz  and IV}, \\
& \leq \frac{2 \kappa^2 J}{n} \; \; \max_{\vv, \vv'} \| \mLambda^{-1} -  {\mLambda'}^{-1} \|_F + \frac{4\kappa^2 J}{n} \frac{\sqrt{J}}{\mu_{min}(\mLambda)}, \mbox{ since $\| \vw_n \|_2^2 \leq \frac{1}{n^2}\kappa^2 J$}, \\
& \leq \frac{4\kappa^2 J \sqrt{J} B^2}{(n-1) \|\mu_{min}(\mM_{\gamma_n})\| } + \frac{4\kappa^2 J}{n} \frac{\sqrt{J}}{\mu_{min}(\mLambda)}, \mbox{ due to inequality III}.\\
& \leq \frac{4\kappa^2 J \sqrt{J} B^2}{(n-1) \gamma_n } + \frac{4\kappa^2 J}{n} \frac{\sqrt{J}}{\gamma_n} \\
& = \frac{4\kappa^2 J \sqrt{J} }{n \gamma_n } \left( 1 + \frac{ \kappa^2 J}{n-1} \right)
\end{align} 
Here, the regularization parameter $\lambda_n$ is the lower bound on the minimum singular values of the matrices $\Lambda$ and $\mM_{\lambda}$.
%where here $\mu_{min}(\mLambda) > B^2$. Hence if we take $\lambda>B^2$, then we have that the final sensitivity is given by:
%$$
% \leq \frac{4 \kappa^2 J^{3/2}}{n} \left( \frac{B^2}{ \lambda - B^2} + \frac{1}{\lambda} \right)
%$$
Hence the final sensiitvity of the data can be upper bound by $\frac{4\kappa^2 J \sqrt{J} }{n \gamma_n } \left( 1 + \frac{ \kappa^2 J}{n-1} \right)$.
%\leon{@mijung, can you check?}
%

%Since 
%$\sigma_{min}(\mLambda)/B^2 \geq 1+ \rho \geq \rho \geq 0$,
%
%$1/\rho \geq B^2 /\sigma_{min}(\mLambda)$
 %i.e., $B=\frac{1}{\sqrt{n}}\kappa \sqrt{J}$, 
%the final sensitivity is bounded by 
%GS_2(s_n) \leq  \frac{4 \kappa^2 J^3}{n^2}\cdot\frac{1} {\sigma_{min}(\mLambda)} + \frac{4\kappa^2 J}{n}\cdot \frac{\sqrt{J}}{\sigma_{min}(\mLambda)} =  \left(\frac{4 (\kappa^2 J)^2}{n^2} + \frac{2\kappa^2 J \sqrt{J}}{n^2} \right)\cdot \frac{1}{\sigma_{min}(\mLambda)}.
%\end{align}
%What all this tedious algebra tells us is that the sensitivity of the test statistic is roughly inversely proportional to the minimum singular value of the empirical 2nd-moment matrix, and proportional to the square of l2-norm bound on $\frac{1}{\sqrt{n}}\vz_i$.   
%\mijung{So, if we use a large dataset and $n$ is much larger than the inverse of the smallest signular value, the noise won't be too much.}
%

The inequalities we used are given by
\begin{itemize}
    \item I: Due to the Sherman–Morrison formula, we can re-write 
    \begin{align}\vw_n{\trp}(\mLambda - \frac{n}{n-1}\vw_n \vw_n{\trp})^{-1}\vw_n = 
    \vw_n{\trp}\mLambda^{-1}\vw_n + \frac{\frac{n}{n-1}(\vw_n{\trp}\mLambda^{-1}\vw_n)^2}{1+\frac{n}{n-1}\vw_n{\trp}\mLambda^{-1}\vw_n}.
    \end{align} Now, we can bound 
    \begin{align}
        &\left|\vw_n{\trp}(\mLambda - \frac{n}{n-1}\vw_n \vw_n{\trp})^{-1}\vw_n - {\vw_n}{\trp}(\mLambda'-\frac{n}{n-1}{\vw_n}{\vw_n}{\trp})^{-1}{\vw_n} \right| \nonumber \\\leq &|\vw_n{\trp}\mLambda^{-1}\vw_n-{\vw_n}{\trp}\mLambda'^{-1}{\vw_n}| 
        + \left|\frac{\frac{n}{n-1}(\vw_n{\trp}\mLambda^{-1}\vw_n)^2}{1+\frac{n}{n-1}\vw_n{\trp}\mLambda^{-1}\vw_n} -\frac{\frac{n}{n-1}({\vw_n}{\trp}\mLambda'^{-1}{\vw_n})^2}{1+\frac{n}{n-1}{\vw'_n}{\trp}\mLambda'^{-1}{\vw'_n}}\right|,\nonumber \\
         &\leq 2|\vw_n{\trp}\mLambda^{-1}\vw_n-{\vw_n}{\trp}\mLambda'^{-1}{\vw_n}|,
    \end{align} where the last line is due to $\vw_n{\trp}\mLambda^{-1}\vw_n\geq \frac{(\vw_n{\trp}\mLambda^{-1}\vw_n)^2}{1+\vw_n{\trp}\mLambda^{-1}\vw_n} \geq0$, and ${\vw_n}{\trp}\mLambda'^{-1}{\vw_n}\geq \frac{({\vw'_n}{\trp}\mLambda'^{-1}{\vw'_n})^2}{1+{\vw'_n}{\trp}\mLambda'^{-1}{\vw'_n}}\geq0$.
    Let  $a = \vw_n{\trp}\mLambda^{-1}\vw_n$ and $b ={\vw'_n}{\trp}\mLambda'^{-1}{\vw'_n}$, then:
    
    \begin{align*}
     A = \left|\frac{a^2}{1+a} - \frac{b^2}{1+b}\right| = \left |\frac{a^2 - b^2+a^2b -b^2a}{(1+a)(1+b)}\right| &= \left |\frac{(a-b)(a+b) + (a-b)ab}{(1+a)(1+b)}\right|\\&= \left|\frac{(a-b)[(a+b) +ab]}{(1+a)(1+b)}\right|
     \end{align*}
    and then we have that:
    $$ A = \left|\frac{(a-b)[(1+a)(1+b)-1]}{(1+a)(1+b)}\right| \leq |a-b|$$
    Hence,  $\left|\frac{\frac{n}{n-1}(\vw_n{\trp}\mLambda^{-1}\vw_n)^2}{1+\frac{n}{n-1}\vw_n{\trp}\mLambda^{-1}\vw_n} -\frac{\frac{n}{n-1}({\vw_n}{\trp}\mLambda'^{-1}{\vw_n})^2}{1+\frac{n}{n-1}{\vw'_n}{\trp}\mLambda'^{-1}{\vw'_n}}\right| \leq \left( \frac{n}{n-1}\right) \left( \frac{n-1}{n}\right) |\vw_n{\trp}\mLambda^{-1}\vw_n-{\vw_n}{\trp}\mLambda'^{-1}{\vw_n}| $
    \item II: For a positive semi-definite $\mSigma$, $0 \leq \vm{\trp}\mSigma\vm \leq \|\vm \|^2_2 \| \mSigma\|_F$, where $\| \mSigma \|_F$ is the Frobenius norm.
    \item III: We here will denote $\tilde{\vv} = \sqrt{\frac{n}{n-1}}\vv$ and $\tilde{\vv}' = \sqrt{\frac{n}{n-1}}\vv'$. Due to the Sherman–Morrison formula,
    \begin{align}
    \mLambda^{-1} = 
    (\mM_{\gamma_n})^{-1} - (\mM_{\gamma_n})^{-1}\frac{\tilde{\vv}\tilde{\vv}{\trp} }{1+\tilde{\vv}{\trp}(\mM_{\gamma_n})^{-1}\tilde{\vv}} (\mM_{\gamma_n})^{-1}.
    \end{align}
    For any eigenvectors $\vu_j, \vu_k$ of $\mM{\trp}\mM$, we have
    \begin{align}
    \vu_j{\trp}(\mLambda^{-1}-\mLambda'^{-1})\vu_k = 
    \mu_j^{-1}\mu_k^{-1}\left(\frac{(\vu_j{\trp}\tilde{\vv}) (\tilde{\vv}{\trp} \vu_k) }{1+\tilde{\vv}{\trp}(\mM_{\gamma_n})^{-1}\tilde{\vv}}- \frac{(\vu_j{\trp}\tilde{\vv}') (\tilde{\vv}'{\trp} \vu_k) }{1+\tilde{\vv}'{\trp}(\mM_{\gamma_n})^{-1}\tilde{\vv}'}\right),
    \end{align} 
    where $\mu_j, \mu_k$ are corresponding eigenvalues. Now, we rewrite the Frobenius norm as (since it is invariant under any orthogonal matrix, so we take the one formed by the eigenvectors from $M^\top M$ with this property):
    \begin{align}
    % CHECKED
    \| ( \mLambda^{-1} -  \mLambda'^{-1})  \|^2_F
     &=  \| \Q (\mLambda^{-1} -  \mLambda'^{-1}) \Q{\trp} \|^2_F, \nonumber \\
    &= \sum_{j,k}^J \left(\vu_j{\trp} (\mLambda^{-1} -  \mLambda'^{-1}) \vu_k\right)^2, \nonumber \\
    &\leq \frac{2}{(1+\tilde{\vv}{\trp}(\mM_{\gamma_n})^{-1}\tilde{\vv})^2}\sum_{j,k}^J\frac{(\vu_j{\trp}\tilde{\vv})^2 (\tilde{\vv}{\trp} \vu_k)^2 }{\mu_j^2\mu_k^2} \nonumber \\
    & \quad + \frac{2}{(1+\tilde{\vv}'{\trp}(\mM_{\gamma_n})^{-1}\tilde{\vv}')^2}\sum_{j,k}^J\frac{(\vu_j{\trp}\tilde{\vv}')^2 (\tilde{\vv}'{\trp} \vu_k)^2 }{\mu_j^2\mu_k^2},  \\ & \qquad \qquad \qquad \qquad \qquad \qquad \mbox{ [due to } \|a-b\|_2^2 \leq 2\|a\|_2^2+2\|b\|_2^2], \nonumber\\
    &\leq \frac{2}{(\tilde{\vv}{\trp}(\mM_{\gamma_n})^{-1}\tilde{\vv})^2}\sum_{j,k}^J\frac{(\vu_j{\trp}\tilde{\vv})^2 (\tilde{\vv}{\trp} \vu_k)^2 }{\mu_j^2\mu_k^2} \nonumber \\
    & \quad + \frac{2}{(\tilde{\vv}'{\trp}(\mM_{\gamma_n})^{-1}\tilde{\vv}')^2}\sum_{j,k}^J\frac{(\vu_j{\trp}\tilde{\vv}')^2 (\tilde{\vv}'{\trp} \vu_k)^2 }{\mu_j^2\mu_k^2}, \\
    &\leq \frac{\mu_{min}(\mM_{\gamma_n})^2}{B^4J} \sum_{j,k}^J\frac{2((\vu_j{\trp}\tilde{\vv})^2 (\tilde{\vv}{\trp} \vu_k)^2 + (\vu_j{\trp}\tilde{\vv}')^2 (\tilde{\vv}'{\trp} \vu_k)^2 )}{\mu_j^2\mu_k^2} \\
     &\leq \frac{(n-1)^2\mu_{min}(\mM_{\gamma_n})^2}{n^2 B^4J} \sum_{j,k}^J\frac{2((\vu_j{\trp}\tilde{\vv})^2 (\tilde{\vv}{\trp} \vu_k)^2 + (\vu_j{\trp}\tilde{\vv}')^2 (\tilde{\vv}'{\trp} \vu_k)^2 )}{\mu_{min}(\mM_{\gamma_n})^4} \\
    &\leq \frac{(n-1)^2 2 J}{n^2\mu_{min}(\mM_{\gamma_n})^2 B^4}(\| \tilde{\vv}\|_2^8+\| \tilde{\vv}'\|_2^8), \\
     &\leq \left(\frac{n}{n-1}\right)^2\frac{ 4 J}{\mu_{min}(\mM_{\gamma_n})^2}B^4, \\
     %&\leq \frac{4 J}{(\mu_{min}(\mLambda)- B^2)^2}B^4, \\
    %&\leq \frac{2}{(\rho B^2)^2}\sum_{j,k}^J \left[(\vu_j{\trp}\vv) (\vv{\trp} \vu_k)+ (\vu_j{\trp}\vv') (\vv'{\trp} \vu_k)\right], \\
    %
    %&\leq \frac{2 J^2}{(\rho B^2)^2}(\| \vv\|_2^4+\| \vv'\|_2^4), \\
    %
    %&\leq \frac{4 J^2}{(\rho B^2)^2} B^4 = \frac{4J^2}{\rho^2}.
    \end{align}
    Note that we can get equation $26$ by noticing that:
    \begin{align*}
     \tilde{\vv} {\trp} (\mM_{\gamma_n})^{-1} \tilde{\vv} \leq \|\tilde{\vv} \|^2_2 \| (\mM_{\gamma_n})^{-1} \|_F &\leq \left(\frac{n}{n-1}\right) B^2 \sqrt{\frac{1}{\mu_1^2(\mM_{\gamma_n})} + \dots + \frac{1}{\mu_{min}^2(\mM_{\gamma_n})}} \\ &\leq \left(\frac{n}{n-1}\right) \frac{B^2 \sqrt{J}}{\mu_{min}(\mM_{\gamma_n})}
    \end{align*}
    %\leon{Here we have used that $$\mu_min(\mLambda) - B^2 > 0$}
    %
    
    \item IV: 
    \begin{align*}
    \max_{\vv,\vv'} \big| \vw_n^\top \mLambda^{-1} \vw_n  - {\vw'_n}^\top \mLambda'^{-1} {\vw'_n}\big| \leq 
\max_{\vv,\vv'}\Big[ \big| {\vw'_n}^\top \mLambda^{-1} {\vw'_n}  - {\vw'_n}^\top \mLambda'^{-1} {\vw'_n}\big| \\ + \big| \vw_n^\top \mLambda^{-1} \vw_n  - {\vw'_n}^\top \mLambda^{-1} {\vw'_n}\big| \Big]   \end{align*}

We write $\vw_n^\top \mLambda^{-1} \vw_n = \left( \mLambda^{-1/2}\vw_n\right)^\top \left( \mLambda^{-1/2}\vw_n\right) $ and similarly ${\vw'_n}^\top \mLambda^{-1} {\vw'_n} = \left( \mLambda^{-1/2}{\vw'_n}\right)^\top \left( \mLambda^{-1/2}{\vw'_n}\right)$.\\
\begin{align*}
    \left| \vw_n^\top \mLambda^{-1} \vw_n  - {\vw'_n}^\top \mLambda^{-1} {\vw'_n} \right | &= \left | \left( \mLambda^{-1/2}\vw_n\right)^\top \left( \mLambda^{-1/2}\vw_n\right) - \left( \mLambda^{-1/2}{\vw'_n}\right)^\top \left( \mLambda^{-1/2}{\vw'_n}\right) \right |\\
    &= \left|\left( \mLambda^{-1/2}\vw_n + \mLambda^{-1/2}{\vw'_n}\right)^\top\left(\mLambda^{-1/2}\vw_n - \mLambda^{-1/2}{\vw'_n} \right)\right| \\
    &= \left|\left(\mLambda^{-1/2}\left(\vw_n + {\vw'_n}\right) \right)^\top\left(\mLambda^{-1/2}\left(\vw_n - {\vw'_n}\right) \right)\right| \\
    &\leq \left\| \mLambda^{-1/2}\left(\vw_n +{\vw'_n} \right)\right\|_2 ~\left\| \mLambda^{-1/2}\left(\vw_n -{\vw'_n} \right)\right\|_2 \\
    &\leq \left\| \mLambda^{-1/2}\left(\vw_n +{\vw'_n} \right)\right\|_2 ~ \frac{\kappa \sqrt{J}}{n} \left\|  \mLambda^{-1} \right\|_F^{1/2} \text{using equality (II)} \\
    &\leq \frac{2\kappa^2 J}{n^2} \left\|  \mLambda^{-1} \right\|_F,\\
    &= \frac{2\kappa^2 J}{n^2} \sqrt{\frac{1}{\mu^2_1(\mLambda)} + \cdots + \frac{1}{\mu^2_{min}(\mLambda)}},\\
    &\leq \frac{2\kappa^2 J}{n^2} \frac{\sqrt{J}}{\mu_{min}(\mLambda)}.
    %&\leq \frac{2\kappa^2 J}{n^2} \sqrt{J} (\rho+1) B^2, \mbox{ since } \sigma_{min}(\mLambda) \geq \rho B^2 + \| \vv\|^2_2 \geq (\rho+1) B^2 >0
\end{align*}
where the last equality comes from that $\mLambda$ is real and symmetric.
\end{itemize}

\section{Other Possible Ways to Make the Test Private} \label{app:other_ways}
\subsection{Perturbing the Kernel Mean in RKHS } \label{app:rkhs_noise}

In \cite{chaudhuri2011differentially}, the authors proposed a new way to make the solution of the regularized risk minimization differentially private by injecting the noise in objective itself. That is :
$$ f_{priv} = \arg\min \big(J(f,\xv) + \frac{1}{n} \bv^\top f \big) $$
However, it is not an easy task to add perturbation in functional spaces. The authors in \cite{hall2013differential} proposes to add a sample path from gaussian processes into the function to make it private. 
\begin{lemma}[Proposition 7 \cite{hall2013differential}] Let $G$ be a sample path of a Gaussian process having mean zero and covariance function $k$. Let $K$ denote the Gram matrix \textit{i.e.} $K = [k(\xv_i, \xv_j)]_{i,j=1}^n$. Let $\{ f_D : D \in \mathcal{D} \}$ be a family of functions indexed by databases. Then the release of :
$$\tilde{f}_D = f_D + \frac{\Delta c(\beta)}{\alpha} G$$
is $(\alpha, \beta)$-differentially private (with respect to the cylinder $\sigma$-field $F$) where $\Delta$ is the upper bound on
\begin{equation}
\label{eq:Deltabound}
\sup_{D\sim D'}\sup_{n\in \mathbb N}\sup_{x_1,\ldots,x_n} \sqrt{\left(\bff_D-\bff_{D'}\right)^\top K^{-1}\left(\bff_D-\bff_{D'}\right)}
\end{equation}
and $c(\beta)\geq \sqrt{2\log\frac{2}{\beta}}$.
\end{lemma}

Now, we consider the optimization problem given for MMD and inject noise in the objective itself.  The optimization problem then becomes :
\begin{align*}
{d}_{priv}(p,q) &= \sup_{f \in \ch,~\| f\|_{\ch}\leq 1} \Big[\mathbb{E}_{\xv \sim p } [f(\xv)] - \mathbb{E}_{\xv \sim q } [f(\xv)]  + \big\langle f, g(\Delta,\beta,\alpha) G \big\rangle\Big] \\
&= \sup_{f \in \ch,~\| f\|_{\ch}\leq 1} \Big[ \big\langle f ,\mu_p - \mu_q \big\rangle  + \big\langle f, g(\Delta,\beta,\alpha) G \big\rangle\Big] \\
&= \sup_{f \in \ch,~\| f\|_{\ch}\leq 1} \Big[ \Big\langle f ,\mu_p - \mu_q  + g(\Delta,\beta,\alpha) G  \Big\rangle \Big] \\
&= \|\mu_p - \mu_q  + g(\Delta,\beta,\alpha) G\|_{\ch}
\end{align*}
In the similar way, one get the empirical version of the perturbed MMD distance just by replacing the true expectation with the empirical one.
The problem with a construction above where embedding is injected with a Gaussian process sample path with the same kernel $k$ is that the result will not be in the corresponding RKHS $\mathcal H_k$ for infinite-dimensional spaces (these are well known results known as Kallianpur's 0/1 laws), and thus MMD cannot be computed, i.e. while $f_D$ is in the RKHS, $\tilde f_D$ need not be. This has for example been considered in Bayesian models for kernel embeddings \cite{FlaSejCunFil2016}, where an alternative kernel construction using convolution is given by:
\begin{equation}
 r(x,x')=\int k(x,y)k(y,x')\nu(dy),
\end{equation}
where $\nu$ is a finite measure. Such smoother kernel $r$ ensures that the sample path from a $GP(0,r)$ will be in the RKHS $\mathcal H_k$.

The key property in \cite{hall2013differential} is Prop. 8, which shows that for any $h\in \mathcal H_k$ and for any finite collection of points $\bfx = (x_1,\ldots,x_n)$:

$$\bfh^\top K^{-1}\bfh \leq \Vert h \Vert_{\mathcal H_k}^2.$$

which implies that we only require $\sup_{D\sim D'}\Vert f_D-f_{D'} \Vert_{\mathcal H_k} \leq \Delta$ to hold to upper bound \eqref{eq:Deltabound}. However, in nonparametric contexts like MMD, one usually considers permutation testing approaches. But this is not possible in the case of private testing as one would need to release the samples from the null distribution.

\subsection{Adding $\chi^2$-noise to the Test Statistics}\label{app:chi_square_noise}
Since the unperturbed test statistics follows the $\chi^2$ distribution under the null, hence it is again natural to think to add noise sampled from the chi-square distribution to the test statistics $s_n$. The probability density function for chi-square distribution with $k-$degree of freedom is given as :
\begin{align*}
    f(x,k) = \begin{cases}
    \frac{x^{\frac{k}{2}-1} \exp({-\frac{x}{2}})}{2^{\frac{k}{2}}\Gamma(\frac{k}{2})} , & \text{if $x \geq 0$.}.\\
   0, & \text{otherwise}.
    \end{cases}
\end{align*}
For $k = 2$, we simply have $f(x) = \frac{\exp(-\frac{x}{2})}{2}, ~if~x\geq 0$. As we have been given $s_n = n \vw_n \mSigma_n^{-1}\vw_n$ which essentially depends on $\vz_i ~\forall i \in [n]$. Now, we define $s_n'$ which differs from $s_n$ at only one sample \textit{i.e.} $s_n'$ depends on $\vz_1, \cdots \vz_{i'} , \cdots \vz_n$.We denote $\Delta = s_n - s_n'$. The privacy guarantee is to bound the following term :
\begin{align}
    &\frac{p\big(s_n + x = s_n + x_0\big)}{p\big(s_n + x = s_n + x_0\big)} = \frac{p\big( x = x_0\big)}{p\big( x = s_n'-s_n + x_0\big)} \label{eq:eq_1_case2_priv}  \\
    &= \frac{\exp\Big(-\frac{x_0}{2}\Big)}{\exp\Big(-\frac{s_n'-s_n + x_0}{2}\Big)} = \exp\Big(-\frac{s_n-s_n'}{2}\Big) \leq \exp\Big( \frac{GS_2}{2}\Big) \label{eq:privacy_final}
\end{align}
Hence, we get the final privacy guarantee by equation~\eqref{eq:privacy_final}. But the problem to this approach that since the support for chi-square distributions are limited to positive real numbers. Hence the distribution in the numerator and denominator in the equation~\eqref{eq:privacy_final} might have different support which essentially makes the privacy analysis almost impossible in the vicinity of zero and beyond. Hence, to hold equation~\eqref{eq:privacy_final}, $x_0$ must be greater than $s_n - s_n'$ for all two neighbouring dataset which essentially implies $x_0 > GS_2(s_n)$. Hence, we get no privacy guarantee at all when the test statistics lies very close to zero. 

However, proposing alternate null distribution  is simple in this case. As sum of two chi-square random variable is still a chi-square with increased degree of freedom. Let $X_1$ and $X_2$ denote $2$ independent random variables that follow these chi-square distributions :
\begin{align*}
    X_1 \sim \chi^2(r_1) ~~~\text{and} ~~~
    X_2 \sim \chi^2(r_2)
\end{align*}
then $Y = (X_1 + X_2) \sim \chi^2(r_1 + r_2)$. Hence, the perturbed statistics will follow chi-square random variable with $J+2$ degree of freedom.

\subsection{Adding Noise to $\mSigma_n^{-1/2}\vw_n$} \label{app:sigam_half}
 One might also achieve the goal to make test statistics private by adding gaussian noise in the quantity $\sqrt{n}\mSigma_n^{-1/2} \vw_n$ and finally taking the $2-$norm of the perturbed quantity. As we have done the sentitivity analysis of $\vw_n^\top \mSigma^{-1} \vw_n$ in the theorem~\ref{thm:test_sn}, the sensitivity analysis of $\sqrt{n}\mSigma^{-1/2}\vw_n$ can be done in very similar way.  Again from the application of slutsky's theorem, we can see that asymptotically the perturbed test statistics will converge to the true one. 
However, similar to section~\ref{sec:null_analysis}, we approximate it with the other null distribution which shows more power experimentally under the noise as well.
Suppose we have to add the noise $\veta\sim \mathcal{N}(0,\sigma^2(\epsilon,\delta_n))$ in the $\mSigma_n^{-1/2}\vw_n$ to make the statistics $s_n$ private. The noisy statistics is then can be written as $$
\tilde{s}_n = \sqrt{n} \left( \mSigma_n^{-1/2} \vw_n + \veta \right)^\top \sqrt{n}\left( \mSigma_n^{-1/2} \vw_n + \veta \right)$$
Eventually, $\tilde{s}_n$ can bewritten as the following : $\tilde{s}_n =  \left(\widetilde{\mSigma_n^{-1/2} \vw_n}\right)^\top \mathbf{A}~ \left(\widetilde{\mSigma_n^{-1/2} \vw_n}\right)$ 
where 
\begin{align}
      \widetilde{\mSigma_n^{-1/2} \vw_n} = \begin{pmatrix} \label{eq:noise_mean_vec_2}
         \sqrt{n} \mSigma_n^{-1/2} \vw_n \\
         \sqrt{n} \frac{\veta}{\sigma(\epsilon,\delta_n)}
    \end{pmatrix}
  \end{align}
  
 $\widetilde{\mSigma_n^{-1/2} \vw_n}$ is a $2J$ dimensional vector. The corresponding covariance matrix $\hat{\mSigma}_n$ is an identity matrix $\mathbf{I}_{2J}$ of dimension $2J\times 2J$. Hence, under the null $\widetilde{\mSigma_n^{-1/2} \vw_n} \sim \mathcal{N}(0, \mathbf{I}_{2J})$.
  We define one more matrix which we call as $\mathbf{A}$ which is 
\begin{align}
    \mathbf{A} = \begin{bmatrix} \label{eq:cov_modified_2}
          \mathbf{I}_J && \mathbf{V} \\
          \mathbf{V} && \mathbf{V}^2
    \end{bmatrix} \text{ where } \mathbf{V} ~=~\text{Diag}(\sigma(\epsilon,\delta_n))
\end{align}

By definition matrix $A$ is a symmetric matrix  which essentially means that there exist a matrix $\mathbf{H}$ such that $\mathbf{H}^\top  \mathbf{A} \mathbf{H} = diag(\lambda_1, \lambda_2 \cdots \lambda_r)$ where $\mathbf{H}^\top \mathbf{H} = \mathbf{H} \mathbf{H}^\top = \mathbf{I}_J$. Now if we consider a random variable $\mathbf{N}_2 \sim \mathcal{N}(0,\mathbf{I}_2J)$ and $\mathbf{N}_1 = \mathbf{H} \mathbf{N}_2$ then following holds asymptotically :
\begin{align*}
     \left(\widetilde{\mSigma_n^{-1/2} \vw_n}\right)^\top \mathbf{A}~ \left(\widetilde{\mSigma_n^{-1/2} \vw_n}\right) \sim (\mathbf{N}_2)^\top \mathbf{A}~(\mathbf{N}_2) \sim (\mathbf{H} \mathbf{N}_2)^\top \mathbf{A}~(\mathbf{H} \mathbf{N}_2)\sim \sum_{i=1}^r  \lambda_i \chi_1^{2,i}
\end{align*}

%\end{proof}

As a short remark, we would like to mention that the in this approach the weights for the weighted sum of $\chi^2$-random variable are not directly dependent on the data which is essentially a good thing from the privacy point of view. Sensitivity of $\mSigma_n^{-1/2}\vw_n$ can be computed in a similar way as in Theorem~\ref{thm:test_sn}. 

\section{Perturbed Samples Interpretation of Private Mean and Co-Variance}\label{app:per_samples}
In order to define differential privacy, we need to define two neighbouring dataset $\mathcal{D}$ and $\mathcal{D}'$. Let us consider  some class of databases $\mathcal{D}^N$ where each datset differ with another at just one data point. Let us also assume that each database carries $n$ data points of dimension $d$ each. Now if we privately want to release data then we consider a function $f:\mathcal{D}^N \rightarrow \mathbb{R}^{nd}$ which simply takes all $n$ data points of the database and vertically stack them in one large vector of dimension $nd$. It is not hard to see now that :
\begin{align}
    GS_2(f) = \sup_{\mathcal{D},\mathcal{D}'} \|f(\mathcal{D}) - f(\mathcal{D}') \|_2 \approx \mathcal{O}(diam(\mathcal{X}))
\end{align}
where $diam(\mathcal{X})$ denotes the input space.  Since the sensitive is way too high (of the order of diameter of input space), the utility of the data is reduced by a huge amount after adding noise in it. 

Here below now we discuss the perturbed sample interpretation of private mean and co-variance. That is to anylyze what level of noise added directly on samples itself would follow the same distribution as private mean. From Lemma~\ref{lem:app_mean2sample}, we see that the variance of the noise come out to be much more tractable in private mean case than adding noise directly to samples. 

\begin{lemma}\label{lem:app_mean2sample}
Let us assume that  $\sqrt{n}\tilde{\vw}_n  = \sqrt{n} \vw_n + \eta$ where $\eta \sim \mathcal{N}(0,\frac{c}{n})$ for any positive constant $c$, $\sqrt{n} \vw_n = \frac{1}{\sqrt{n}}\sum_{i=1}^n \vz_i$ and $\vz_i$s are \textit{i.i.d} samples. Then $\sqrt{n} \vw_n \rightarrow \frac{1}{\sqrt{n}}\sum_{i=1}^n \tilde{\vz}_i$ where $\tilde{\vz}_i = \vz_i + \zeta$ and $\zeta \sim \mathcal{N}(0,\sigma_p^2)$ if $\sigma_p^2 = \frac{c}{n}$
\end{lemma}
%\leon{There are a couple of typos with = (instead of convergence??) and missing tilde, if I think you are trying to show what you are trying to show, then yep it is correct:)}
\begin{proof}
It is  easier to see that $\mathbb{E}~\left[\sqrt{n}\tilde{\vw}_n\right] = \sqrt{n} \vw_n = \mathbb{E}~\left[ \frac{1}{\sqrt{n}}\sum_{i=1}^n \tilde{\vz}_i \right]$. \\
Now, we try to analyze the variance of both the term. 
\begin{align*}
    \frac{c}{n} = n \frac{\sigma_p^2}{n}
\end{align*}
Hence, $\sigma_p^2 = \frac{c}{n}$
\end{proof}

Now similar to lemma~\ref{lem:app_mean2sample}, we want to translate the noise added in the covariance matrix to the sample case. The empirical covaraince matrix $\mSigma_n = \frac{1}{n-1}\sum_{i=1}^n (\vz_i - \vw_n)(\vz_i - \vw_n)^\top$. For now, if we say $(\vz_i - \vw_n) = \hat{\vz}_i$, then $\mSigma_n = \sum_{i=1}^n \frac{\hat{\vz}_i}{\sqrt{n-1}} \frac{\hat{\vz}_i}{\sqrt{n-1}}^\top$. Now, adding a gaussian noise in each $\hat{\vz}_i$ results in the following :

\begin{align*}
    \hat{\mSigma}_n &= \sum_{i=1}^n \left(\frac{\hat{\vz}_i}{\sqrt{n-1}} + \eta_i\right)\left(\frac{\hat{\vz}_i}{\sqrt{n-1}} + \eta_i\right)^\top ~~~~~\text{where } \eta_i \sim \mathcal{N}(0,\sigma^2(\epsilon,\delta_n)) \\
    &=\sum_{i=1}^n \left(\frac{\hat{\vz}_i}{\sqrt{n-1}} \frac{\hat{\vz}_i}{\sqrt{n-1}}^\top + \frac{\hat{\vz}_i}{\sqrt{n-1}} \eta_i ^\top + \eta_i \frac{\hat{\vz}_i}{\sqrt{n-1}}^\top + \eta_i \eta_i^\top\right)
\end{align*}

As can be seen by the above equations, we have similar terms like adding wishart noise in the covariance matrix with $2$ extra cross terms. Hence instead of using the matrix $\tilde{\mSigma}_n$, one can use $\hat{\mSigma}_n$ for $\mSigma_n$ to compute the weights for the null distribution \textit{i.e.} weighted sum of chi-square in section~\ref{sub:null_tcmc}.

\section{Additional experimental Details}
\label{sec:experiments_appendix}

We see that indeed the Type I error is approximately controlled at the required level for TCMC, TCS and NTE algorithm, for both versions of the test, as shown in Figure \ref{fig:type_1}, note that here we allow some leeway due to multiple testing. Again, we emphasis that using the asymptotic $\chi^2$ distribution naively would provide inflated Type I error as shown in Figure \ref{fig:asym}. %it is noted that although NTE-SCF-Asym seems to control type I error for large $\epsilon$, this was shown to have no power on the Blobs dataset. 

In Figure \ref{fig:reg}, we show the effect of the regularisation parameter $\gamma_n$ on the TCS algorithm performance in terms of Type I error and power on the SG, GMD and GVD datasets. For simplicity, we take $\gamma_n = \gamma$ here, rather then let it depend on the sample size $n$. From the results, we can see that if the  $\gamma$ to be too small, we will inject too much noise, and hence we will lose power. Note that any $\gamma>0$ will provide us differential privacy, however if we choose it to be too large, our null distribution will now be mis-calibrated, hurting performance. Hence, there is a trade off between calibration of the null distribution and also the level of noise you need to add. 

%In Figure \ref{fig:noise_type}, we observe that using the asymptotic $\chi^2$ distribution as the null distribution would give you inflated Type I error, regardless of which mechanism is used. This is however not the case, when using the approximation to the finite null distribution that we describe in this paper, the Type I error is controlled at the desired level. The power analysis of the GMD and GVD dataset shows that the Wishart mechanism I and Wishart mechanism III in general perform the best, however the performance is very similar overall. Note that for the GVD dataset sampling locations provide better performance, hence the NTE algorithm provides better power than the TCMC, as shown in Figure \ref{fig:vary_main}. Note here, we do not consider the Wishart mechanism II, due to its restriction that $\epsilon$ lies in $(0,1)$.

\begin{figure}[ht!]
\centering
\includegraphics[width=0.75\textwidth]{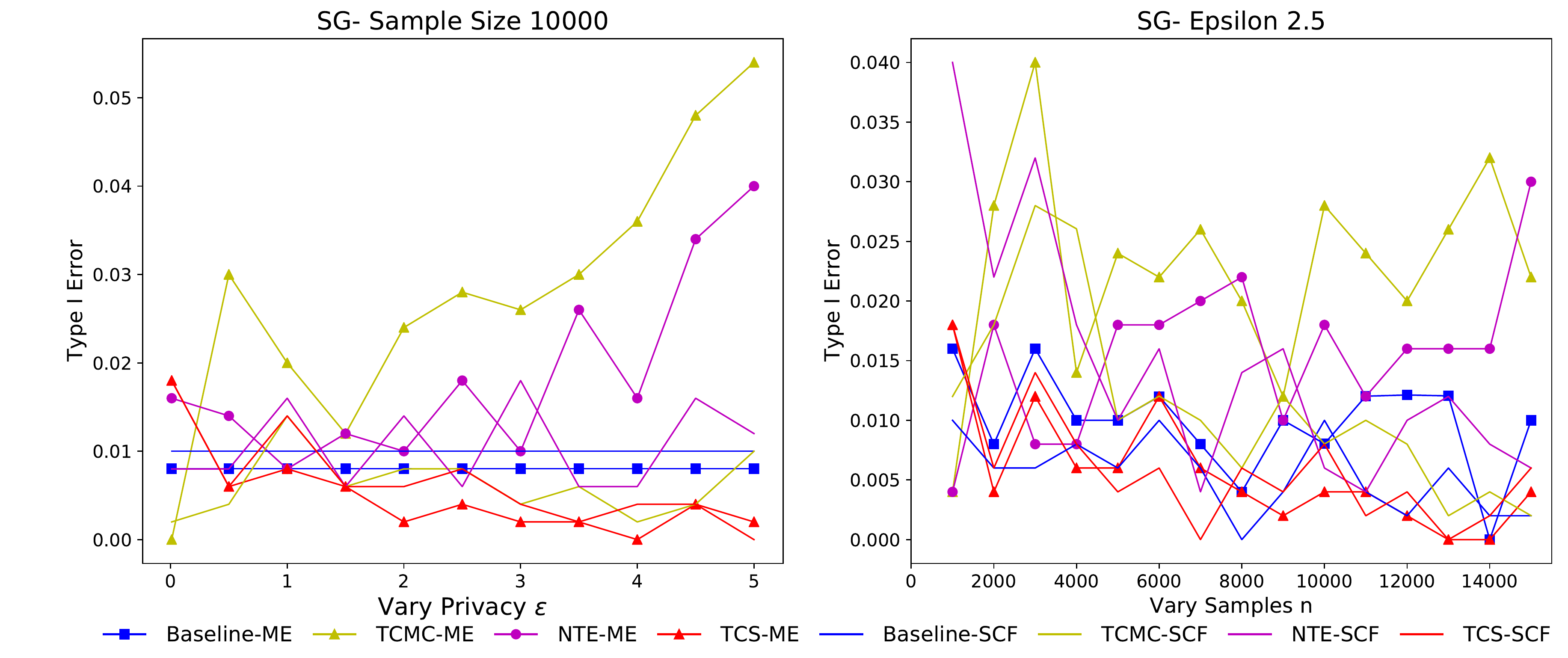}
\caption{Type I Error for the SG Dataset, with baselines ME and SCF, $\delta=1e-5$. \textbf{Left:} Vary $\epsilon$, fix $n=10000$ \textbf{Left:} Vary $n$, fix $\epsilon=2.5$}
\label{fig:type_1}
\end{figure}

\begin{figure}[ht!]
\centering
\includegraphics[width=\textwidth]{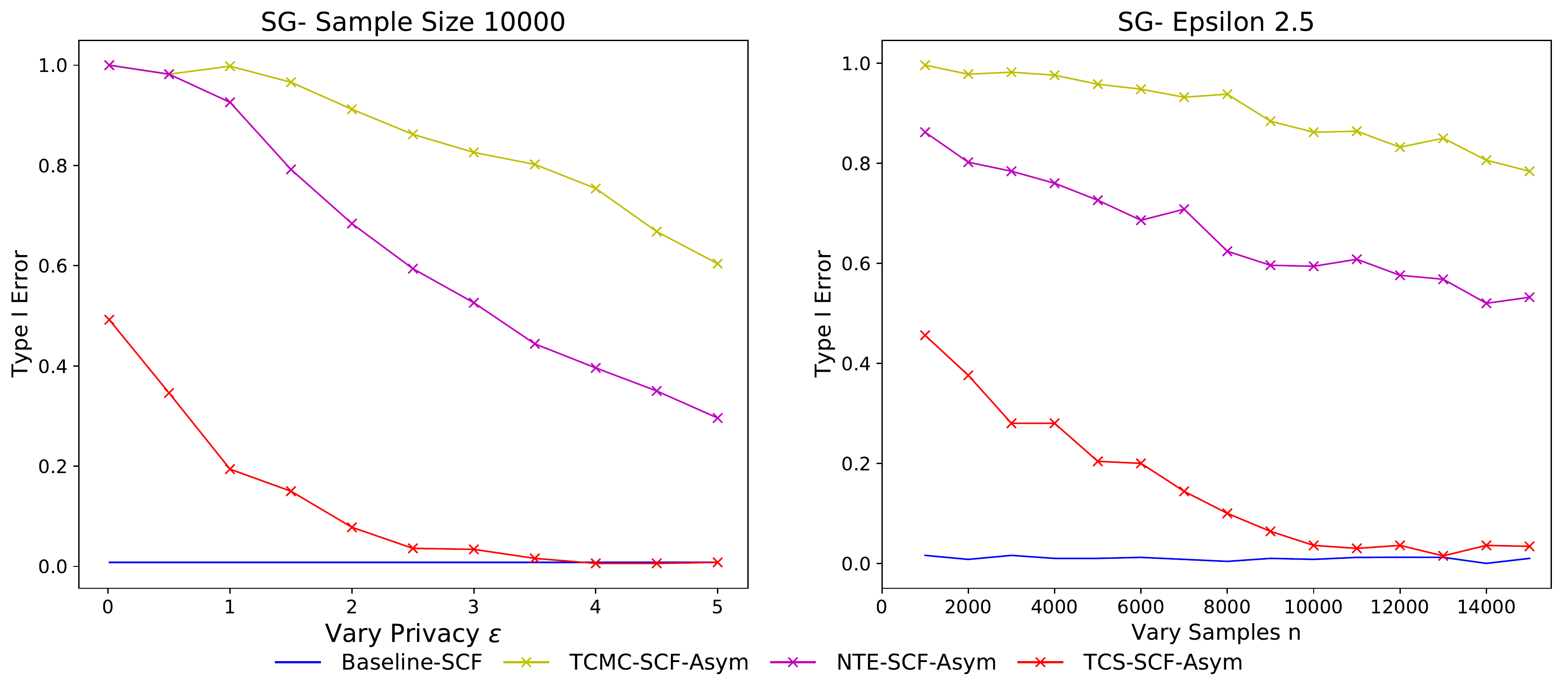}
\caption{Type I error for the SCF versions of the test, using the asymptotic $\chi^2$ distribution as the null distribution.}
\label{fig:asym}
\end{figure}

\begin{figure}[ht!]
\includegraphics[width=\textwidth]{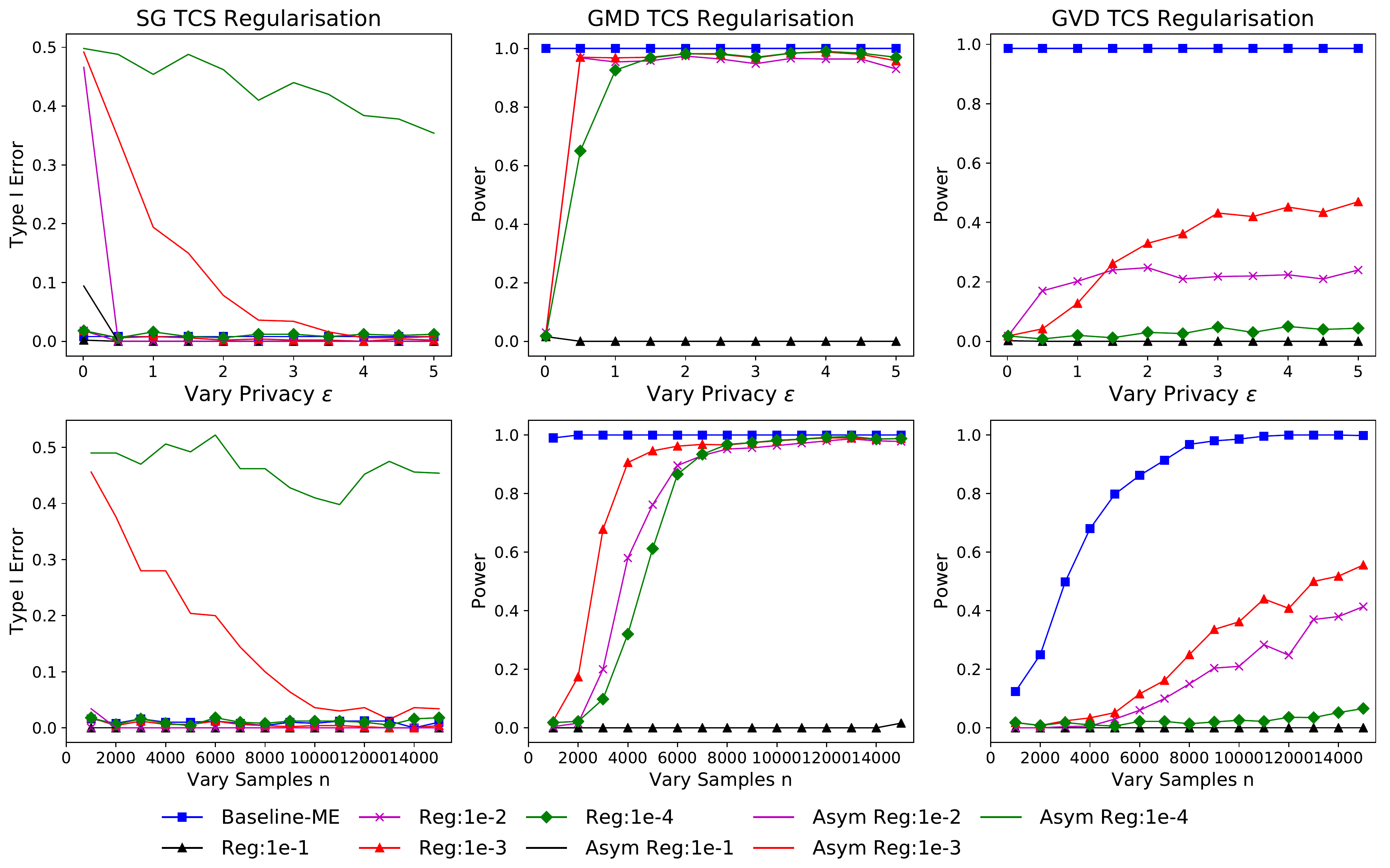}
\caption{Type I error for the SG dataset , Power for the GMD, GVD dataset over $500$ runs, with $\delta=1e-5$ for the TCS algorithm with different regularisations. 
    \textbf{Top}: Varying $\epsilon$ with $n=10000$. \textbf{Bottom}: Varying $n$ with $\epsilon=2.5$. Here Asym * represents using the asymptotic $\chi^2$ null distribution.}
\label{fig:reg}
\end{figure}
\section{Proof of Theorem 5.1}\label{sec:proof_theorem_5_1}

\begin{proof}
The variance $\sigma_\vn^2$ of the zero-mean noise term $\vn$ added to the mean vector $\vw_n$ is of the order $\mathcal{O}(\frac{1}{n^2})$. Hence the variance of $\sqrt{n}\vn$ is of the order  $\mathcal{O}(\frac{1}{{n}})$. According to Slutsky's theorem, $\sqrt{n}\tilde{\vw}_n$ and $\sqrt{n}\vw_n$ thus converge to the same limit in distribution, which under the null hypothesis is $\mathcal N(0,\mSigma)$, with $\mSigma=\mathbb E\left[\vz\vz^\top\right]$. Similarly, the eigenvalues of the covariance matrix corresponding to the Wishart noise to be added in $\mSigma_n$ are also of the order $\mathcal{O}(\frac{1}{n})$ which implies that $\tilde{\mSigma}_n + \gamma_n I$ and $\mSigma_n + \gamma_n I$ converge to the same limit, i.e. $\mSigma$. Therefore, $\tilde{s}_n$ converges in distribution to the same limit as the non-private test statistic, i.e. a chi-squared random variable with $J$ degrees of freedom.
\end{proof}
 \clearpage
 
 \end{document}